\PassOptionsToPackage{table,xcdraw}{xcolor}
\documentclass{article}
\usepackage{iclr2027_conference,times}
\iclrfinalcopy
\usepackage{amsmath,amsfonts,bm}

\def\eqref#1{equation~\ref{#1}}

\def\1{\bm{1}}

\DeclareMathAlphabet{\mathsfit}{\encodingdefault}{\sfdefault}{m}{sl}
\SetMathAlphabet{\mathsfit}{bold}{\encodingdefault}{\sfdefault}{bx}{n}

\usepackage{hyperref}
\usepackage{url}

\usepackage{bm}
\usepackage{array}     
\usepackage{booktabs}
\usepackage{amsfonts}
\usepackage{amsmath}
\usepackage{amssymb}
\usepackage{amsthm}
\usepackage{multirow}
\usepackage{nicefrac}
\usepackage{microtype}
\usepackage{xcolor}
\usepackage{graphicx}
\usepackage{enumitem}
\usepackage{mathtools}
\usepackage{graphicx}
\usepackage{subcaption}
\usepackage{threeparttable}
\usepackage{placeins}

\usepackage{algorithm}
\usepackage{algpseudocode}

\usepackage{titletoc}

\newtheorem{theorem}{Theorem}
\newtheorem{lemma}[theorem]{Lemma}

\newtheorem{corollary}[theorem]{Corollary}
\theoremstyle{definition}
\newtheorem{definition}[theorem]{Definition}
\theoremstyle{remark}
\newtheorem{remark}[theorem]{Remark}

\newcounter{assumption}
\renewcommand{\theassumption}{$\mathfrak{A}$\arabic{assumption}}

\newcommand{\asm}[1]{%
  \refstepcounter{assumption}%
  \textbf{\theassumption}%
  \label{#1}%
}

\title{Trust Your Targets: Is Target Propagation Block-Diagonal Gauss-Newton?}

\author{Giuseppe Castiglione, Christopher Buckley \& Ivor Simpson  \\
Department of Informatics\\
University of Sussex\\
\texttt{\{G.Castiglione,C.L.Buckley,I.Simpson\}@sussex.ac.uk} \\
}

\begin{document}

\maketitle

\begin{abstract}
Target propagation (TP) is a backpropagation alternative in which local targets carry the global learning signal. Prior works claim that TP performs Gauss-Newton (GN) optimisation on these targets, but the correspondence has only been established under unrealistic assumptions where hitting them exactly is trivial.  Practical variants instead descend on a least-squares problem via a single gradient step. In this paper, we rederive TP from a trust region under an implicit, previously unexamined norm, and introduce CGTP (Conjugate Gradient Target Propagation), which uses conjugate gradients to solve the layer-local subproblems in parallel. A single adaptive parameter jointly controls this trust region and a matching bound on the weight update; KFAC emerges exactly as the special case where the per-sample Jacobians are replaced by their batch mean, giving its factored damping a variational grounding. On MNIST and Fashion-MNIST MLPs, CGTP matches the accuracy of ADAM and KFAC while needing orders of magnitude fewer outer iterations than ADAM (about 10x fewer backward-pass-equivalents overall), while remaining stable across depth.

\end{abstract}

\section{Introduction}
\label{sec:intro}

End-to-end backpropagation paired with first-order optimizers \citep{Kingma2014AdamAM} is the de facto standard for training neural networks. Despite its remarkable effectiveness, backpropagation-based learning  imposes a sequential dependency whereby each layer's gradient depends on every subsequent layer's. Consequently, error signals must propagate fully backward before updates can begin, which is known as the \emph{update locking} problem \citep{Jaderberg2016DecoupledNI}. 
This dependency increasingly limits layer-parallel training \citep{wang2021infopro}, and has motivated a family of \emph{local learning} methods that allow layers to perform optimization steps independently against local objectives. These methods span a spectrum: greedy supervised approaches \citep{Belilovsky2018GreedyLL} achieve full layer independence but suffer from \emph{information collapse} \citep{wang2021infopro}. Conversely, predictive coding methods \citep{millidge2020predictive} simulate global feedback through iterative inference, justified by the free-energy principle \citep{Friston2009PredictiveCU}.  Target propagation (TP) \citep{Lee2014DifferenceTP} sits between these global communication extremes: local targets are inferred via a single pass through an approximate inverse, which in principle enables parallel layerwise optimization.  However, to realize the benefits of this decoupled optimization, the local objectives must be theoretically sound. While greedy and predictive-coding methods can be derived from information-theoretic or variational principles \citep{Tishby2015DeepLA, Friston2009PredictiveCU}, TP's primary theoretical justification is through a loose analogy to Gauss-Newton (GN) optimization \citep{Meulemans2020ATF} that is built on assumptions that frequently break down in practice.  

The existing proofs that connect TP to GN rely on several assumptions that diverge drastically from practice: they assume batch sizes of 1, single-layer updates, a regression objective, strictly invertible layers, and an \emph{infinitesimal} step size $\beta$. Two of these are particularly damaging. First, the infinitesimal-$\beta$ requirement defeats the purpose of a second-order method, which is sought \emph{for} its faster convergence.   Second, the assumption of strictly invertible layers forces practical TP variants to increasingly rely on inexact inverses, via autoencoders \citep{Lee2014DifferenceTP}, or in the extreme case fixed random feedback weights \cite{Shibuya2022FixedWeightDT}. This requires expensive feedback-weight training loops and more iterations to converge.  Relying on many propagation phases especially runs counter to the goals of a local algorithm: since the targets are derived through a propagation phase, we should endeavour to minimize the number of propagation phases altogether.

The result is a striking mismatch: TP is theoretically motivated as approximate GN, yet its dominant practical variants distance themselves from GN at every step. The local loss - the squared-error reconstruction objective each layer minimizes against its target - is treated as a means to an end: a single gradient step toward a target whose theoretical role is only loosely justified.  We argue that this is a missed opportunity: if TP is genuinely a form of Gauss-Newton, then taking the local subproblem seriously
- ie reaching the minimum - should yield the convergence advantages over first-order methods seen in modern BP-based second-order optimizers (e.g., K-FAC \citep{Martens2015OptimizingNN}, SHAMPOO \citep{shampoo}), but crucially, without the update-locking bottleneck. 

In this paper, we give TP a rigorous grounding as a sequence of local trust-region subproblems, which we then solve with a second-order method. Our specific contributions are as follows:
\begin{itemize}
  \item \textbf{(C1) A rigorous derivation of TP under the trust-region framework.} 
  Prior TP-GN equivalence proofs rely on a set of restrictive assumptions, typically left implicit or deferred to appendices.  We collect these explicitly (Table~\ref{tab:assumptions}), and relax each in turn.
 We derive TP from a linearly constrained least-squares subproblem via a linearized change of variables, accommodating  classification objectives, multi-layer updates, finite step sizes, and non-invertible layers (Sections~\ref{sec:loss}-\ref{sec:multi-layer}). Two decoupling theorems - one via norm choice, one via a simplicial representation - yield interpretable, local subproblemss whose unconstrained optima recover the regularized-inverse TP update.
  \item \textbf{(C2) A conjugate gradient solver, with KFAC as a special case.} KKT analysis of the constrained per-layer subproblem shows that the targets must be fit in the curvature-induced metric, not the Euclidean one, explaining why a naive second-order fit of the TP loss fails (Section~\ref{sec:KKT-conditions}). The null-space method reduces each layer to a positive definite system in the weights, solved in parallel by CG after a single backward pass (Section~\ref{sec:cg-alg}). Averaging the per-sample curvatures recovers KFAC with factored damping exactly (Corollary~\ref{cor:kfac}).
  \item \textbf{(C3) Proof-of-concept validation.} On MNIST and Fashion-MNIST, CGTP matches Adam, KFAC and the strongest  TP variants in accuracy with orders of magnitude fewer outer iterations (Table~\ref{tab:main-results}), though at higher cost per iteration (Table~\ref{tab:resource-comparison}), while remaining stable across depth (Figure~\ref{fig:depth-ablation}).
\end{itemize}

\section{Identifying the Theoretical Foundations of Target Propagation}\label{sec:prelim}

\subsection{Preliminary: Classical Target Propagation}\label{sec:conventional-tp}

In this section, we briefly review the common presentation of TP.  Consider a dataset $\mathcal{D} = \{(x_n, y_n)\}_{n=1}^N$ of $N$ samples. We fit an $L$-layer feedforward multilayer perceptron (MLP) computing:
\begin{equation}
u^{(0)}_n = x_n; \quad z^{(\ell+1)}_n = W^{(\ell)\top}u^{(\ell)}_n + b^{(\ell)}_n; \quad  u^{(\ell+1)} = \sigma(z^{(\ell+1)})    
\end{equation}
The output $z^{(\mathrm{out})}_n = W^{(L)\top} u^{(L)}_n + b^{(L)}$ is trained to minimize the per-sample loss $\mathcal{L}_n(z^{(\mathrm{out})}_n, y_n)$. To assign credit, TP begins by defining a target for the output via a gradient step on the loss:
\begin{equation}
    \hat{z}^{(\mathrm{out})}_n(\beta) = z^{(\mathrm{out})}_n - \beta \frac{\partial \mathcal{L}_n}{\partial z^{(\mathrm{out})}} \coloneqq  z^{(\mathrm{out})}_n - \beta r_n, \label{eq:infer-head-target}
\end{equation}
where $\beta$ is a step size and $r_n$ denotes the (generalized) residual. The composite mapping $f^{(\ell)}:z^{(\ell)}\to z^{(\ell+1)}$ are then approximately inverted to compute intermediate targets $\hat{z}^{(\ell)}_n(\beta) = f^{(\ell)+}(\hat{z}^{(\ell+1)}_n(\beta))$.  These targets define a layer-local regression problem, from which parameter gradients are derived:
\begin{equation}
    L^{(\ell)}_{\text{TP}}(W^{(\ell)}, b^{(\ell)}) = \frac{1}{\beta}\sum_n ||  W^{(\ell)\top}u^{(\ell)}_n + b^{(\ell)}_n - \hat{z}^{(\ell+1)}_n(\beta)||^2 \label{eq:tp-local-loss}
\end{equation}

\subsection{Classical Target Propagation as Hybrid Gauss Newton}\label{sec:naive-gn}

\begin{table}[t]
  \caption{Assumptions invoked by prior proofs of the TP-GN equivalence}
  \label{tab:assumptions}
  \centering
  \small
  \begin{tabular}{ll}
    \toprule
    \textbf{Assumption} & \textbf{Relaxed in} \\
    \midrule
       \asm{asm:step}: Infinitesimal step size $\beta$ &  Section~\ref{sec:trust-region}: Trust-region radius \\
       \asm{asm:loss}: $\ell_2$ regression loss  & Section~\ref{sec:loss}: Generalized surrogate loss \\
    \asm{asm:single}: Single-layer updates & Section~\ref{sec:multi-layer}: Norm-penalties \\
    \asm{asm:invert}: Strictly invertible layers & Section~\ref{sec:multi-layer}: Norm-penalties, Section~\ref{sec:simplicial-loss}: Simplicial decoupling \\
    \asm{asm:batch}: Mini-batch of size 1  & Section~\ref{sec:KKT-conditions}: Constrained activation space \\
    \bottomrule
  \end{tabular}
\end{table}

The connection between the GN method and TP, as presented in such works as \citep{JaderbergCOVGK16, Meulemans2020ATF, Roulet2021TargetPV} requires  several assumptions to draw formal equivalence, which we collect in Table~\ref{tab:assumptions}. Let $F^{(\ell',\ell)}$ denote the composite map from $z^{(\ell)}$ to $z^{(\ell')}$, with Jacobian $J^{(\ell',\ell)}$. GN is a second-order method applied to least-squares problems (\ref{asm:loss}).  Consider the following objective, for inferrring the optimal target of a single (\ref{asm:single}) intermediate layer $z^{(l)}_n$ for a single datapoint (\ref{asm:batch}):
\begin{equation}
  \mathcal{L}^{(\ell)}_n = \frac{1}{2} \big\| \hat{z}^{(\mathrm{out})}_n - F^{(\mathrm{out}, \ell)}(z^{(\ell)}_n) \big\|^2,
  \label{eq:gn-objective-z}
\end{equation}
A GN step proceeds by linearizing $F^{(\mathrm{out},\ell)}$ around the current $z^{(\ell)}_n$ and solving the resulting linear least-squares problem. Writing the residual as $r_n = F^{(\mathrm{out},\ell)}(z^{(\ell)}_n) - \hat{z}^{(\mathrm{out})}_n$, the GN update is
\begin{equation}
  [\Delta z^{(\ell)}_{\text{GN}}]_n = \big( J^{(\mathrm{out},\ell)\top} J^{(\mathrm{out},\ell)} \big)^{+} J^{(\mathrm{out},\ell)\top} r.
  \label{eq:gn-step-z}
\end{equation}
If $F$ is invertible (\ref{asm:invert}), then so is the Jacobian, and the above simplifies to:
\begin{equation}
  [\Delta z^{(\ell)}_{\text{GN}}]_n =[J^{(\mathrm{out},\ell)}_n]^{-1} r_n.
  \label{eq:gn-step-z-invertible}
\end{equation}

Let us return to the TP loss from Equation~\ref{eq:tp-local-loss}, considered as a function of the intermediate activation $z^{(\ell)}_n$.  In the limit that the step-size $\beta \to 0$ (\ref{asm:step}), a Taylor expansion on the TP gradient yields:
\begin{align}
[\Delta z^{(\ell)}_n]_{\mathrm{TP}}
&= z^{(\ell)}_n - [F^{(out, \ell)}]^{-1}(z^{(out)}_n - \beta r_n) \\
&\approx z^{(\ell)}_n - [F^{(out, \ell)}]^{-1}(z^{(out)}_n) + \beta \frac{\partial [F^{(out, \ell)}]^{-1}(z^{(out)}_n)}{\partial z^{(out)}_n} r_n = \beta\frac{\partial [F^{(out, \ell)}]^{-1}(z^{(out)}_n)}{\partial z^{(out)}_n} r_n
\end{align}
By the inverse function theorem, the Jacobian of $G$ is the inverse of the Jacobian of $F$.  As such:
\begin{equation}
    [\Delta z^{(\ell)}_n]_{\mathrm{TP}} =  \beta [J^{(out, \ell)}_n]^{-1} r_n = \beta [\Delta z^{(\ell)}_{\text{GN}}]_n \label{eq:tp-is-gn},
\end{equation}
which demonstrates the equivalence.  Crucially, this is \emph{not} a GN step on the weights. TP as presented in \citep{Meulemans2020ATF} is a hybrid algorithm: a GN step infers local targets, and a gradient descent step converts those targets into a weight update. The theory for a full TP-based GN update on the weights~\citep{BengioIteratingApprox, Meulemans2020ATF} has been developed only in the single-datapoint case (\ref{asm:batch}), where any inferred target
is exactly realizable and the TP step coincides with a full GN step up to a
positive scalar. The remainder of this paper relaxes each of these assumptions in
turn.

\section{Generalising Target Propagation via Trust-Region Methods}\label{sec:trust-region}

The classical presentation of TP uses a single parameter ($\beta$) to control both the the approximation fidelity and step size.  Trust-Regions (TR) methods separate these roles, permitting accelerated optimization even with poor approximations.  One constructs a local quadratic model $m$ of the objective, takes a step that approximately minimizes $m$ within a radius $\Delta$, and then accepts/rejects the step based on the ratio $\rho$ between the predicted reduction (in $m$) and the reduction of the true loss. The radius adapts based on $\rho$: a successful step extends it, a failed step contracts it. 

In what follows we generalise TP within this framework, and relax the remaining assumptions in the proof of  Section~\ref{sec:naive-gn}. Section~\ref{sec:loss} identifies the loss that TP approximates, clarifying its role as a surrogate for general loss functions (\ref{asm:loss}), and how its fidelity is measured through the trust-region ratio $\rho$. Section~\ref{sec:multi-layer} demonstrates that a principled choice of trust-region metric naturally accomodates the multi-layer setting (\ref{asm:single}) and non-invertible layers (\ref{asm:invert}).

\subsection{Relaxing \ref{asm:loss} via Convex Surrogates}\label{sec:loss} 

Adopting the Trust Region framework for TP requires identifying \textit{which} objective it is trying to approximate.  We propose that the TP loss in Equation~\ref{eq:tp-local-loss} can be viewed as a surrogate of the true objective $\mathcal{L}$, insofar as (see Supp~\ref{sec:refined-tp}):\begin{equation}
    ||\nabla_{z^{(\ell)}} \mathcal{L}_{\text{TP}}||^2 = 0 \implies ||\nabla_{z^{(\mathrm{out})}} \mathcal{L}||^2 = 0
\end{equation}
The key observation is that the objectives used to train neural networks are commonly convex functions of the outputs.  As such, the vanishing of $||\nabla_{z^{(out)}} \mathcal{L}||$ is a sufficient condition for optimality of the network, motivating the following general surrogate loss:
\begin{equation}
    \mathcal{J}_n = ||\nabla_{z^{(out)}} \mathcal{L}(z^{(out)}(x_n, y_n; \theta)||_{M_n}^2 \label{eq:surrogate-func}
\end{equation}
Here, $M_n \succeq 0$ denotes a choice of norm appropriate to the problem.    Letting $r_n \coloneqq \nabla_{z^{(out)}} \mathcal{L}_n$ denote the generalized residual, we consider the result of applying a perturbation $\Delta z^{(\ell)}_n$ to the pre-activation of layer $\ell$ at the point $n$.  Taylor expanding, we have:
\begin{equation}
    r_n(z^{(\ell)}_n + \Delta z^{(\ell)}_n) \approx r_n + H_n J^{(out,\ell)}_n \Delta z^{(\ell)}_n; \quad H_n\coloneqq \nabla_{z^{(out)}_n}^2 \mathcal{L}_n\label{eq:output-hess-defn}
\end{equation}
We then obtain the following subproblem:

\begin{definition}[Generalized Gauss-Newton Subproblem for the Pre-Activations]
\label{def:single-pt-single-layer-subproblem}  For a general objective $\mathcal{L}$, the surrogate loss $\mathcal{J}_n$ may be approximated as a weighted least-squares subproblem:
\begin{equation}
    \mathcal{J}_n \approx 
     ||M_n^{1/2}(r_n + H_n J^{(out,\ell)}_n \Delta z^{(\ell)}_n)||_{}^2 \coloneqq m_n(\Delta z^{(\ell)}_n) \label{eq:single-pt-single-layer-subproblem}
\end{equation}
The improvement ratio $\rho$ is then calculated by comparing $\Delta \mathcal{J}_n$ (measured) and $\Delta m_n$ (predicted).
\end{definition}

\subsection{Relaxing \ref{asm:single} and \ref{asm:invert} via a Principled Choice of Trust-Region Metric}\label{sec:multi-layer}

Prior works generalize the single-layer result of Definition~\ref{def:single-pt-single-layer-subproblem} to the multi-layer setting by appealing to the common empirical approximation of a block-diagonal Gauss-Newton matrix \citep{Meulemans2020ATF}. In this section, we show that achieving layer decoupling requires no such approximation - only a principled choice of metric.  The key observation is that the multi-layer problem is typically underdetermined.  Let $J_n$ denote the stacked Jacobian and  $\Delta z_n$ the stacked perturbation, as follows:
\begin{equation}
  J_n \coloneq \big[\, I \;\big|\; J^{(\mathrm{out}, L)}_n \;\big|\; \cdots \;\big|\; J^{(\mathrm{out}, 1)}_n \,\big]; \quad \Delta z_n \coloneq \big[\, \Delta z^{(\mathrm{out})\top}_n \;\big|\; \Delta z^{(L)\top}_n \;\big|\; \cdots \;\big|\; \Delta z^{(1)\top}_n \,\big]^\top
  \label{eq:stacked-vars}
\end{equation}
Henceforth we drop the $\mathrm{out}$, denoting $J_n^{(\ell)} \coloneqq J^{(\mathrm{out}, \ell)}_n$.  The natural generalization of Definition~\ref{def:single-pt-single-layer-subproblem} is:
\begin{equation}
  m_n(\Delta z_n) = \;\big\| r_n + H_n J_n \Delta z_n \big\|_{M_n}^2 \label{eq:multi-layer-target-prop}
\end{equation}
Note that, if each $\Delta z^{(\ell)}_n \in \mathbb{R}^K$, the joint perturbation across layers introduces $(L+1)K$ unknowns against only $C \ll K$ residuals, and so there are infinitely many solutions.  In the trust region setting, a unique solution may be fixed via a norm constraint $||\Delta z_n||_{B_n} < \Delta$, for choice of metric $B_n$.  The following theorem highlights a specific family of norms that natively decouples the layers:

\begin{theorem}[Egalitarian Norm-Induced Layer Separation]
\label{thm:uniform-decoupling-property}

Let $\lambda$ be a Lagrange multiplier, and
consider the penalty form of the norm constraint for metric $B_n$:
\begin{equation}
    m_n(\Delta z_n; \lambda) = \tfrac{1}{2}\big\| r_n + H_n J_n \Delta z_n \big\|_{M_n}^2 + \tfrac{\lambda }{2}\Delta z_n^{\top}B_n\Delta z_n  \qquad B_n =  \bigoplus_{\ell=1}^{L+1} B^{(\ell)}_n\label{eq:regularised-multilayer-subproblem}
\end{equation}
Then, if the metric is egalitarian (ie satisfies $\sum_{\ell} J^{(out,l)}_n [B^{(l)}_n]^{-1} [J^{(out,l)}]_n^{\top} = I$), the unconstrained solution for the layer perturbation is:
\begin{equation}
    \Delta z^{(\ell)}_n = -\tfrac{1}
    {L} J^{(\ell)+}_n\tilde{r}_n(\lambda) = -\tfrac{1}
    {L}[B^{(\ell)}_n]^{-1}J^{(\ell)\top}_n\tilde{r}_n(\lambda)
    \label{eq:regularized-target-def}
\end{equation}
Where $\tilde{r}_n(\lambda)$ denotes the regularized target, and is given by:
\begin{equation}
    \tilde{r}_n(\lambda) = \big(H_n^{\top}M_nH_n + \lambda I\big)^{-1}H_n^{\top}M_n r_n \label{eq:regularized-residual-def}
\end{equation}
\end{theorem}
\begin{proof}
    The identiy block in $J_n$ makes it surjective (Lemma~\ref{lem:surjectivity-of-J}), so that for any $B_n$ the 
    the normal equations factor as $(\mathrm{operator}) J_n \Delta z_n = (\mathrm{target})$.  The egalitarian block requirement is then chosen to coincide with the classical TP update.  Full details in Supp~\ref{appendix:proof-of-uniform-norm}. 
\end{proof}

\begin{remark}[Recovering Classical TP as the Low-Trust Limit in an Egalitarian Newton Norm] As $\lambda \to \infty$, choosing penalty norm $M_n = H_n^{-1}$, the regularized target of \ref{eq:regularized-target-def} approaches the Classical TP target of \ref{eq:tp-is-gn} with $\beta = \lambda^{-1}$.

\end{remark}

\begin{remark}[Interpretation of the Egalitarian Norm]
Theorem~\ref{thm:uniform-decoupling-property} has a natural reading: no single layer should dominate the others in its contribution to the regularized objective. This is structurally similar to the scale-invariant diagonal damping used in classical Levenberg-Marquardt (LM) \citep{nocedal2006numerical}.  Furthermore, the pseudo-inverse leaves unperturbed those direction of the pre-activations that the layer discards.  Supp~\ref{appendix:proofs:non-invertible} makes this construction explicit for the output layer.
\end{remark}

\section{Beyond Classical TP: Second-Order Methods on the Local Subproblem}
\label{sec:second-order-tp}

Having relaxed Assumptions~\ref{asm:step}-\ref{asm:invert}, we now address \ref{asm:batch}, and the prospect of deriving GN updates from the TP loss.  The hybrid route of prior works (a single gradient step on the weights) is naturally read as the low-trust regime of our framework.  Where previous approaches assess (and adapt) approximate inverses through reconstruction, the trust-region formulation calibrates trust through \emph{utility}.  A second-order step is licensed when reductions in the subproblem align with reduction in the true loss.  It then seems a simple matter to apply a second-order optimiser to~\eqref{eq:tp-local-loss}: the Hessian is block diagonal, and the Newton step yields the familiar least-squares fit.  We will show in this section that doing so is not straightforward. The primary difficulty is one of feasibility:~\eqref{eq:tp-local-loss} assesses attainability of the local targets a Euclidean norm.  When \ref{asm:batch} holds, the targets are exactly realisable and so the choice of metric is irrelevant.  For larger batch sizes, this choice of proves decisive.  

To this end, in Section~\ref{sec:KKT-conditions}, we extend our trust-region subproblem to incorporate feasibility constraints.  The KKT conditions reveal the necessity of measuring fit using the metric induced by curvature.  The numerical cost of the KKT solution motivates a majorization-minimization decomposition of that metric (Section~\ref{sec:simplicial-loss}), yielding new local losses that generalize~\eqref{eq:tp-local-loss} and admit second-order optimizers (Section~\ref{sec:cg-alg}). We then damp the weight updates to combat overfitting (Section~\ref{sec:weight-decay}), recovering KFAC as a corrollary when samples are replaced by their batch average.

\subsection{KKT Conditions for the Constrained Subproblem}
    \label{sec:KKT-conditions}

The natural generalisation of the subproblem of Equation~\ref{eq:multi-layer-target-prop} is to simply sum $m_n$ across datapoints. This constitutes $N$
independent least-squares problems, whose minimisers are available in closed form via Theorem~\ref{thm:uniform-decoupling-property}. However, such solutions ar not necessarily feasible: the model architecture constrains the space of pre-activations to satisfy $\Delta Z^{(\ell)} = \tilde{U}^{(\ell - 1)}\Delta W^{(\ell-1)}$ for some $\Delta W^{(l-1)}$ (here, $\sim$ indicates augmentation by ones).  Note, the constraint requires only that $\Delta W^{(\ell-1)}$ exists: it does not commit us to any particular choice.  Hence, to retain the pre-activations as the principal optimisation variables, we must incorporate the following constraint:

\begin{lemma}[Feasibility of perturbations (layer version)]
\label{lem:feasibility}
Let $\tilde{U}^{(\ell-1)} \in \mathbb{R}^{N \times d_{\ell-1}}$ and $\Delta Z^{(\ell)} \in \mathbb{R}^{N \times d_\ell}$. There exists $\Delta W^{(\ell-1)} \in \mathbb{R}^{d_{\ell-1} \times d_\ell}$ satisfying $\Delta Z^{(\ell)} = \tilde{U}^{(\ell-1)} \Delta W^{(\ell-1)}$ if and only if each column of $\Delta Z^{(\ell)}$ lies in $\mathrm{col}(\tilde{U}^{(\ell-1)})$. Equivalently,
\begin{equation}
  \Pi^{(\ell-1)} \, \Delta Z^{(\ell)} = 0,
  \label{eq:feasibility}
\end{equation}
where $\Pi^{(\ell-1)}$ is the orthogonal projector onto $\mathrm{col}(\tilde{U}^{(\ell-1)})^\perp$.
\end{lemma}
\begin{proof}
The equation $\Delta Z^{(\ell)} = \tilde{U}^{(\ell-1)} \Delta W^{(\ell-1)}$ holds column-wise: each column of $\Delta Z^{(\ell)}$ is a linear combination of the columns of $\tilde{U}^{(\ell-1)}$ with coefficients given by the corresponding column of $\Delta W^{(\ell-1)}$. Such a linear combination exists if and only if the column lies in $\mathrm{col}(\tilde{U}^{(\ell-1)})$, which is equivalent to vanishing under projection onto its orthogonal complement.
\end{proof}

\begin{remark}[Classical TP as Projected Gradient Descent] The classical TP weight update, when applied to the inputs $U^{(\ell-1)}$, is effectively a projection of the targets into $\mathrm{col}(\tilde{U}^{(\ell-1)})$ via $\Delta\bar{Z}^{(\ell)}_{\textrm{PGD}}
  \propto \, \tilde{U}^{(\ell-1)} \tilde{U}^{(\ell-1)\top} \Delta Z^{(\ell)}$ (Supp~\ref{appendix:proof:tp-as-pgd}).
\end{remark}
Lemma~\ref{lem:feasibility} holds at the level of a single layer. In what follows
we require a constraint acting on all layers at once, after stacking across datapoints.  This will keep the per-sample subproblems block-diagonal, but distributes the layer constraint across the stacked vector.  This new constraint is expressed in the following Lemma (proof in Supp~\ref{sec:shared-feasibility}):
\begin{lemma}[Feasibility of perturbations (stacked version)]\label{lemma:stacked-perturbation-contraint}
Let the stacked pre-activations be denoted $\Delta \bm{z} = (\Delta z_1, \dots, \Delta z_N) \in \mathbb{R}^{ND}$, and let $E^{(\ell)}$ be defined such that $\Delta z_n^{(\ell)} = E^{(\ell)} \Delta z_n$.  Then the layer feasibility operators of Lemma~\ref{lem:feasibility} may stated as $\bm{\Pi}\Delta \bm{z} =0$, where
\begin{equation}
\bm{\Pi} = \sum_{\ell} \Pi^{(\ell-1)} \otimes E^{(\ell)\top}E^{(\ell)}.
\end{equation}
\end{lemma}
We are now ready to examine the constrained subproblem. 

\begin{theorem}[Solution of the Joint Constrained Subproblem]
\label{thm:full-joint-subproblem}
Let $G_n \coloneqq J_n^\top H_n^\top M H_n J_n$ denote the Hessian of the per-datapoint subproblem, and $\tilde{G}_n(\lambda) \coloneqq G_n + \lambda B_n$ its regularisation. The unconstrained problems are uncoupled across datapoints, so we define the stacked subproblem Hessian $\bm{\tilde{G}}(\lambda) = \bigoplus_n \tilde{G}_n(\lambda)$. Finally, let $\Delta \bm{z}^\star(\lambda)$ denote the stacked unconstrained minimiser of \eqref{eq:regularized-target-def}. Then, up to an additive constant in $\Delta \bm{z}$, the constrained joint subproblem is:
\begin{equation}
  m(\Delta \bm{z}; \lambda) = 
  \big(\Delta \bm{z} - \Delta \bm{z}^\star(\lambda)\big)^\top \bm{\tilde{G}}\big(\Delta \bm{z} - \Delta \bm{z}^\star(\lambda)\big) \quad \text{subject to} \quad \bm{\Pi}\Delta \bm{z} = 0
  \label{eq:full-joint-subproblem}
\end{equation}
The optimal perturbations are then given by projecting $\Delta \bm{z}^\star(\lambda)$ onto $\ker \bm{\Pi}$ in the metric induced by $\bm{\tilde{G}}(\lambda)$:
\begin{equation}
  \Delta \bm{z}^\star_{\bm{\Pi}}(\lambda) = \Delta \bm{z}^\star(\lambda)  - \bm{\tilde{G}}(\lambda)^{-1} \bm{\Pi}^\top \big(\bm{\Pi}\bm{\tilde{G}}(\lambda)^{-1} \bm{\Pi}^\top\big)^{-1} \bm{\Pi}\Delta \bm{z}^\star(\lambda).
  \label{eq:full-kkt}
\end{equation}
\end{theorem}

The proof follows from KKT analysis (Supp~\ref{app:kkt}). The naive Newton step on Equation~\ref{eq:tp-local-loss}, by contrast, is a \textit{Euclidean} projection: it treats all directions in $\Delta Z$ space on equal footing. When $G_n$ is singular, directions in its null space leave the local loss exactly unchanged, so the constraint can be satisfied by routing the correction through this null space at zero cost; damping replaces "zero cost" with "low cost," graded by proximity to the null space, but the cheap subspace persists. The Euclidean projection has no access to this structure and generically pays for directions the curvature metric would have obtained for free.

Unfortunately, computing \eqref{eq:full-kkt} directly is not feasible. The Schur complement $\bm{\Pi}\bm{\tilde{G}}(\lambda)^{-1}\bm{\Pi}^{\top}$ is dense and of dimension $N(L+1)K \times N(L+1)K$. The product inherits neither the block diagonal structure of $\bm{\tilde{G}}(\lambda)$ nor the Kronecker structure of $\bm{\Pi}$, making the cross-layer couplings unavoidable and thus prohibitive at practical scale.

\subsection{Local Surrogates Through Simplicial Relaxations}\label{sec:simplicial-loss}

Section~\ref{sec:trust-region} recovered the TP update by choosing a
block-diagonal norm, which resolves the under-determination of the multi-layer
subproblem and yields layer-separable targets. The curvature, however, remains
coupled across layers.  This suffices for inferring local targets, but not hitting them with a second-order method. In this section we separate the curvature
itself, replacing the subproblem with a layer-local surrogate that majorises it.

Consider, for each datapoint $n$, a set of $L$ scalars $\alpha_n^{(\ell)} > 0$ with $\sum_{\ell} \alpha_n^{(\ell)} = 1$, and the family of perturbations satisfying $J_n^{(\ell)} \Delta z_n^{(\ell)} = -\alpha_n^{(\ell)} r_n$.  By inspection, this family minimises the subproblem in the absence of constraints: the final layer residual is simply split across the $L$-layers, and the local residual may be fit exactly.  This perspective enables us to identify a principled decomposition which generalises to the constrained setting.  It is encapsulated in the following theorem:
\begin{theorem}[Layer-separable majorization]
\label{thm:majorization}
Let $M \succeq 0$ and let $\alpha_n$ lie in the interior of the simplex. Then the following is a majorizer of $m_n$ at every value of $\alpha_n$
\begin{equation}
\label{eq:jensen-bound}
  m_n(\Delta z) \;\leq\; \frac{1}{2}\sum_{\ell} \Big\| \sqrt{\alpha^{(\ell)}_n}\, r_n
  + \tfrac{1}{\sqrt{\alpha^{(\ell)}_n}} H_nJ^{(\ell)}_n \Delta z^{(\ell)}_n \Big\|_{M_n}^2
  \;\equiv\; \widehat{m}_n(\Delta z; \alpha_n),
\end{equation}
Furthermore, in the presence of a regulariser $\lambda B^{(\ell)}_n$, the following sequence converges to the minimum of the local subproblem:
\begin{equation}
\label{eq:alpha-p-iteration}
    \text{(A)}\quad \Delta z^{(\ell)}_n \leftarrow -\alpha^{(\ell)}_n \big( G^{(\ell)}_n + \alpha^{(\ell)}_n \lambda B^{(\ell)}_n \big)^{-1} g^{(\ell)}_n, \qquad \text{(B)}\quad \alpha^{(\ell)}_n \leftarrow \frac{\lVert \Delta z^{(\ell)}_n \rVert_{G^{(\ell)}_n}}{\sum_k\lVert \Delta z^{(k)}_n \rVert_{G^{(k)}_n}},
\end{equation}

where $G_n^{(\ell)} = J_n^{(\ell)\top} H_n^{\top}M_n H_n J_n^{(\ell)}$ and $g_n^{(\ell)} = J_n^{(\ell)\top} H_n^{\top} M_n r_n$.

\end{theorem}
\begin{proof}
Write $r_n = (\sum_\ell \alpha_n^{(\ell)}) r_n$ and apply Jensen's inequality to $\|\cdot\|_{M_n}^2$. The main result is then alternating minimisation after proving that $ \widehat{m}_n(\Delta z_n; \alpha_n) + \tfrac{1}{2}\lambda ||\Delta z_n||^2_{B_n}$ is convex (Supp~\ref{appendix:proof:simplectic-solutions}).
\end{proof}
Note that \textit{any} $\alpha$ on the simplex yields a valid subproblem (\eqref{eq:jensen-bound}).  We list several approximations in Supp~\ref{sec:simple-simplex-schemes}, all of which. empriically, converge faster than the uniform baseline $\alpha^{(\ell)}_n = L^{-1}$.

\subsection{Null-Space Methods on the Local Surrogate}
\label{sec:cg-alg}

The per-layer surrogate of Theorem~\ref{thm:majorization} is significantly smaller than the full multi-layer subproblem of Theorem~\ref{thm:full-joint-subproblem}, but its $NK \times NK$ Hessian still lacks the structure that would admit the closed-form KKT solution of Supp~\ref{app:kkt}.  While the Schur complement is expensive to form, its components - block-diagonal Hessian $\bm{\tilde{G}}(\lambda)$ and the Kronecker-structured projector $\bm{\Pi}$ - admit cheap matrix-vector products. Hence, methods like conjugate gradient (CG) are a natural choice of solver, as they only require the action of each component ~\citep{nocedal2006numerical}.  Though one may use projected CG to optimise $\Delta Z$ directly, the bi-level approach of the null-space method is more memory efficient.  As in Classical TP, we identify local targets $\Delta Z^{(\ell)\star}(\alpha^{(\ell)})$ via  the inversion in Theorem~\ref{thm:majorization}.  The null-space of the constraint is then simply $\Delta Z^{(\ell)}(\Delta W^{(\ell - 1)}) = U^{(\ell-1)}\Delta W^{(\ell - 1)}$, and we minimise the reduced system, which is a weighted least-squares fit:  
\begin{equation}
\min_{\Delta W^{(\ell-1)}} \frac{1}{2} \big\lVert \tilde U^{(\ell-1)}
\Delta W^{(\ell-1)} - \Delta Z^{(\ell)\star}(\alpha^{(\ell)}) \big\rVert_{\bm{\tilde G}^{(\ell)}}^{2},
\qquad
\bm{\tilde G}^{(\ell)} = \bigoplus_n \tilde G_n^{(\ell)}(\lambda; \alpha^{(\ell)}_n)\label{eq:weighted-least-squares}.
\end{equation}
The gradient is $\sum_n u_n \big( \Delta W^{(\ell-1)\top} u_n -
\Delta z_n^{(\ell)\star} \big)^{\top} \tilde G_n^{(\ell)}$.  To proceed, we define the linear operator $\mathcal{G}^{(\ell)}_{\lambda}(V) \coloneqq \sum_n u_n u_n^{\top} V
\tilde G_n^{(\ell)}(\lambda; \alpha)$, and use CG to solve:
\begin{equation}
\label{eq:reduced-system}
\mathcal{G}^{(\ell)}_{\lambda}\big(\Delta W^{(\ell-1)}\big) =
 -\sum_n u_n g_n^{(\ell)\top}.
\end{equation}

\subsection{Preventing Overfitting with Adaptive Weight Decay}
\label{sec:weight-decay}

Solving the local subproblem to high accuracy is not uniformly beneficial. Though CGTP reduces training loss far faster than first-order baselines, its validation error begins to rise after a small number of outer iterations, leading to a performance gap when compared to BP+ADAM.  It is easy to see the cause: while the trust region bounds the magnitude permissible $\Delta Z^{(\ell)}$, it does nothing to bound $\Delta W^{(\ell-1)}$, which can be arbitrarily large whenever $U^{(\ell-1)}$ is poorly conditioned.  In this section, we suggest an adaptive strategy for controlling the conditioning.
\paragraph{Choice of norm.}
Phase one fixes a metric on pre-activation space: solving the damped subproblem declares that the effective local geometry of $\Delta z^{(\ell)}_n$ is $\bm{\tilde G}^{(\ell)})$ rather than the Euclidean one. Any quantity that maps \emph{into} pre-activation space should be measured accordingly. A weight perturbation $I \otimes \Delta W^{(\ell-1)}$ is a linear map from the metric space $(\mathbb{R}^{N\times K_{\mathrm{in}}}, I)$ to $(\mathbb{R}^{N\times K_{\mathrm{out}}}, \bm{\tilde G}^{(\ell)})$, and its size is its operator (Hilbert-Schmidt) norm between those two inner-product spaces:
\begin{align}
\lVert I\otimes \Delta W^{(\ell-1)} \rVert_{I \to \bm{\tilde G}^{(\ell)}}^{2} 
&= \sum_n\operatorname{trace}(\Delta W^{(\ell-1)\top} \Delta W^{(\ell-1)}\tilde G_n^{(\ell)}(\lambda; \alpha^{(\ell)}_n)) \\
\implies \lVert  \Delta W^{(\ell-1)} \rVert_{\mathrm{HS}}&= \lVert \Delta W^{(\ell-1)} \bar{G}^{1/2} \rVert_{F}^{2}
+ \lambda \lVert \Delta W^{(\ell-1)} \bar{B}^{1/2}\rVert_{F}^{2}
\label{eq:hs-norm}.
\end{align}
When $B^{(\ell)}_n = I$ second term is the ordinary weight decay, while the first term controls a local Lipschitz constant that penalizes overall model sensitivity.
\paragraph{Regularization Strength.}
Each phase now has its own constraint: during inference, $\lVert \Delta Z^{(\ell)} \rVert_{F} \le \delta_z$ is controlled by an LM rule on the reduction ratio, whereas $\lVert \Delta W^{(\ell-1)} \rVert_{\mathrm{HS}} \le \delta_w$ is to be enforced during the second weight-fitting phase.  Independently adapting a second radius would leave us with a two-trust-region problem, which, while being well-studied, introduces additional complexity \citep{, ai2009strong, bienstock2014}.  We instead observe that our constraint sets are both convex and have non-zero intersection ($\Delta W = 0$ lies in the interior of both), and may thus use a single trust-region parameter to limit the weight perturbation to the intersection.  This leads us to the following defintions, and the resulting CG system (Proof in Supp~\ref{sec:proof-thm-inscribed}, full listing in Algorithm~\ref{alg:cg}).  

\begin{theorem}[Single-parameter coupling of the two trust regions]
\label{thm:inscribed}
Fix a layer $\ell$ and suppress the layer index. Let $\delta_z, \delta_w > 0$ be
fixed radii and $\lambda > 0$ a trust parameter, and set
\begin{equation}
\label{eq:mu-from-lambda}
\mu_z = \frac{\lambda}{\delta_z^{2}},
\qquad
\mu_J = \frac{\lambda}{\delta_w^{2}},
\qquad
\tilde G_n(\mu_z; \alpha) = \frac{1}{\alpha_n} G_n + \mu_z I,
\qquad
\bar{\tilde G} = \bar G + \mu_z I,
\end{equation}
with $\bar G = \frac{1}{N} \sum_n \frac{1}{\alpha_n} G_n$, so that the damping
of the pre-activation bound is also the metric of the weight bound. With
\begin{equation}
\label{eq:two-forms}
q_z(\Delta W) = \frac{1}{\delta_z^{2}} \lVert \tilde U \Delta W \rVert_{F}^{2},
\qquad
q_w(\Delta W) = \frac{1}{\delta_w^{2}} \lVert \Delta W \bar{\tilde G}^{1/2}
\rVert_{F}^{2},
\end{equation}
the ellipsoid $\{ q_z + q_w \le 1 \}$ is inscribed in $\{ q_z \le 1 \} \cap
\{ q_w \le 1 \}$, and penalising it with the single multiplier $\lambda$ gives
\begin{align}
\label{eq:single-lambda}
L_\lambda(\Delta W)
&= \sum_n \frac{1}{2} \Big\lVert \sqrt{\alpha_n} \, r_n
+ \frac{1}{\sqrt{\alpha_n}} J_n \Delta W^{\top} u_n \Big\rVert_{M_n}^{2}
+ \frac{\mu_z}{2} \lVert \tilde U \Delta W \rVert_{F}^{2}
+ \frac{\mu_J}{2} \lVert \Delta W \bar{\tilde G}^{1/2} \rVert_{F}^{2} \\
&= \mathrm{const} + \frac{1}{2} \lVert \tilde U \Delta W - \Delta Z^{\star}
\rVert_{\bm{\tilde G}}^{2}
+ \frac{\mu_J}{2} \lVert \Delta W \bar{\tilde G}^{1/2} \rVert_{F}^{2} .
\end{align}
Its unique minimiser solves
\begin{equation}
\label{eq:reduced-system-reg}
\mathcal{G}_{\lambda}(\Delta W) = -\sum_n u_n g_n^{\top},
\qquad
\mathcal{G}_{\lambda}(V) \coloneqq \sum_n u_n u_n^{\top} V
\tilde G_n(\mu_z; \alpha) + \mu_J V \bar{\tilde G} .
\end{equation}
\end{theorem}

\begin{corollary}[General pre-activation norm]
\label{cor:inscribed-B}
Theorem~\ref{thm:inscribed} holds with the pre-activation bound measured in
per-sample norms $B_n \succ 0$, upon replacing $\mu_z I$ by $\mu_z B_n$ in
$\tilde G_n$, $\lVert \tilde U \Delta W \rVert_{F}^{2}$ by
$\sum_n \lVert \Delta W^{\top} u_n \rVert_{B_n}^{2}$ in $q_z$, and
$\bar{\tilde G}$ by $\bar G + \mu_z \bar B$ with
$\bar B = \frac{1}{N} \sum_n B_n$. The CG system becomes
\begin{equation}
\label{eq:reduced-system-reg-B}
\mathcal{G}_{\lambda}(\Delta W) = -\sum_n u_n g_n^{\top},
\qquad
\mathcal{G}_{\lambda}(V) \coloneqq \sum_n u_n u_n^{\top} V
\Big( \frac{1}{\alpha_n} G_n + \mu_z B_n \Big)
+ \mu_J V \big( \bar G + \mu_z \bar B \big) .
\end{equation}
\end{corollary}

\begin{proof}
Identical to that of Theorem~\ref{thm:inscribed}, replacing $I$ by $B_n$ in the
pre-activation penalty and by $\bar B$ in the weight metric. Positive
definiteness of $\mathcal{G}_{\lambda}$ follows from
$\bar G + \mu_z \bar B \succeq \mu_z \bar B \succ 0$.
\end{proof}

\begin{corollary}[Exact recovery of KFAC]
\label{cor:kfac}
Suppose the per-sample metrics are replaced by their batch average,
$\tilde G_n(\mu_z; \alpha) \equiv \bar{\tilde G} = \bar G + \mu_z I$ for all $n$.
Then $\mathcal{G}_{\lambda}$ is Kronecker-structured,
\begin{equation}
\mathcal{G}_{\lambda}(V) = \big( A + \mu_J I \big) \, V \, \big( \bar G + \mu_z I
\big),
\qquad A = \tilde U^{\top} \tilde U,
\end{equation}
and the solution of \eqref{eq:reduced-system-reg} is
\begin{equation}
\label{eq:kfac-recovery}
\Delta W = -\big( A + \mu_J I \big)^{-1} \nabla_{W} L \, \big( \bar G + \mu_z I
\big)^{-1},
\qquad \nabla_{W} L = \sum_n u_n g_n^{\top},
\end{equation}

\end{corollary}

\section{Experiments}\label{sec:experiments}

\begin{table}[t]
\centering
\caption{Error and convergence cost, in backward-pass equivalents. Outer
iterations are counted to the error threshold noted in each column header.   $\times$ denotes failure to reach threshold.  N/A indicates result is not reported by source.}
\label{tab:main-results}
\small
\setlength{\tabcolsep}{4pt}
\begin{tabular}{l rr r@{\hspace{2pt}}r rr r@{\hspace{2pt}}r}
\toprule
& \multicolumn{4}{c}{F-MNIST} & \multicolumn{4}{c}{MNIST } \\
\cmidrule(lr){2-5} \cmidrule(lr){6-9}
alg & final err & outer @ 0.13 & /it & total & final err & outer @ 0.05 & /it & total \\
\midrule
ADAM & $.111\pm.000$ & $25{,}667\pm4509$ & 1 & $25.7$k & $.019\pm.000$ & $7333\pm577$ & 1 & $7.3$k \\
KFAC & $.108\pm.002$ & $2700\pm500$ & 1 & $2.7$k & $.018\pm.001$ & $467\pm58$ & 1 & $467$ \\
\midrule
DRL-lin. & $.111\pm.004$ & N/A & N/A & N/A & $.020\pm.008$ & N/A & N/A & N/A \\
TJML & $.114\pm.001$ & $1468\pm724$ & 1 & $1.5$k & $.018\pm.001$ & $1377\pm756$ & 1 & $1.4$k \\
TPRI & $.152\pm.004$ & $\times$ & $\times$ & $\times$ & $.061\pm.003$ & $\times$ & $\times$ & $\times$ \\
\midrule
GDTP & $.147\pm.000$ & $\times$ & $\times$ & $\times$ & $.092\pm.000$ & $\times$ & $\times$ & $\times$ \\
EucTP & $.146\pm.002$ & $\times$ & $\times$ & $\times$ & $.054\pm.002$ & $\times$ & $\times$ & $\times$ \\
CGTP\,(no WD) & $.113\pm.000$ & $15.0\pm0.0$ & $201$ & $3.0$k & $.037\pm.001$ & $23.3\pm2.9$ & $201$ & $4.7$k \\
CGTP & $.109\pm.002$ & $6.67\pm0.58$ & $201$ & $1.3$k & $.019\pm.001$ & $3.67\pm0.58$ & $201$ & $738$ \\
\bottomrule
\end{tabular}

\end{table}

\paragraph{Datasets and architecture.} We evaluate on two low-dimensional  classification benchmarks: MNIST \citep{MNIST} and Fashion-MNIST \citep{FashionMNIST}. All models are fully connected networks with $\tanh$ nonlinearities, a softmax output layer, and cross-entropy loss. Unless otherwise specified, networks have depth 5 with  256 hidden units.  As an ablation, we consider four variants of trust-region target propagation:
\begin{itemize}[leftmargin=*]
  \item \textbf{GDTP}: Gradient descent step on the Euclidan subproblem (\eqref{eq:classical-tp-as-pgd}). We tune the learning rate, batch size, and Tikhonov dampening for the inverse of $J_n$.
  \item \textbf{EucGTP}: Naive fit of the targets (no curvature).  Same parameters as GDTP, minus learning rate.
  \item \textbf{CGTP (no WD)}(Section~\ref{sec:cg-alg}): CG without weight decay.  An outer iteration costs one backward pass (to form $J_n$, $g_n$) and up to \texttt{max\_iter}$=100$ CG iterations (one JVP + one VJP).
  \item \textbf{CGTP} (Section~\ref{sec:weight-decay}):  automatically adjusted weight decay. 
\end{itemize}
Additionally, we benchmark against five baselines: two general optimisers and three TP variants.
\begin{itemize}[leftmargin=*]
  \item \textbf{BP+Adam}: standard end-to-end training. We tune the learning rate and batch size.
  \item \textbf{KFAC}: \citep{Martens2015OptimizingNN} an approximation of Natural Gradients using Kronecker-Factored approximations. We tune the batch size, the exponential decay weight, and the initial value of the trust-region penalty parameter $\lambda$.
  \item \textbf{DRL} \citep{Meulemans2020ATF}: difference reconstruction loss. A TP variant that learns autoencoder-style inverses with a reconstruction loss specifically to make DTP recover Gauss-Newton targets. We use the same architecture as in \citep{Meulemans2020ATF} and report the published error.
  \item \textbf{TPRI} \citep{Roulet2021TargetPV}: compute targets using a product of layer-wise Tikhonov-regularized Jacobian inversesl weights are updated by SGD on the standard local reconstruction loss. We tune the regularization strength, batch size, learning rate, and momentum.
  \item \textbf{TJML} \citep{Bao2024EfficientTP}: a Jacobian matching loss yields closed-form approximate inverses (transposed weights for linear layers, analytically derived dual activation-function). Closed-form inverse introduces no hyperparameters; we tune the batch size, learning rate, and momentum,
\end{itemize}

Table~\ref{tab:main-results} reports final test error and iteration to milestone for each run, where each method is compared at its own tuned batch size, aggregated over three runs. Note that KFAC and ADAM tend to favour smaller batch sizes, whereas CGTP favours larger ones (as per-sample curvature is only informative when the batch size is large enough to estimate it).  Figure~\ref{fig:ablation-batch-size} shows BP, KFAC, and CGTP at matched batch size, where the order is preserved but the milestone margins change.  More results (such as depth and parameter sensitivity analysis) can be found in Supp~\ref{appendix:more-exp-details}, and more detailed discussions of Related Works can be found in Supp~\ref{sec:related-work}.

\section{Conclusion}
\label{sec:conclusion}

We show that the correspondence between target propagation and Gauss–Newton optimization, previously established only where hitting the targets is trivial, extends to the general case once realization is treated as a trust-region problem in the targets' own metric. Classical TP is a single gradient step under this metric; solving it by conjugate gradient yields CGTP, and sharing curvature across samples recovers KFAC's factored damping exactly. Layers are solved independently and in parallel after one backward pass. On MNIST and Fashion-MNIST MLPs, CGTP matches Adam and KFAC in accuracy with an order of magnitude fewer outer iterations. Our experiments are proof-of-concept, restricted to small MLPs; extending to the stochastic setting, structured architectures, and principled preconditioning for ill-conditioned problems such as PINNs are natural next steps.

\newpage

\section*{AI Tool Declaration}

In this work, we used generative AI tools primarily for fine-tuning and validation: 
\begin{itemize}
\item Derivations were submitted to Claude to check for errors, and fixes devised where they were raised.  Examples abound: the indexing scheme of Lemma~\ref{lem:feasibility} was previously incorrect (the correction verified through a simple numpy script), and the previous definition of $B_n$ in Section~\ref{sec:multi-layer} only held for invertible layers, leading to the correction term in \eqref{eq:choice-of-B-when-not-invertible}).  
\item Our code was profiled and optimised using Claude code tools: for example, the push through identity was adopted in several places, in addition to certain refactors that flagged bottlenecks introduce by logging, and permitted acceleration through use of pytorch compilation tools.  
\end{itemize}

We take responsibility for the final content of this work, including text, claims or artifacts produced with the aid of generative AI.

\bibliography{ref}
\bibliographystyle{iclr2027_conference}

\newpage

\appendix

\newpage
\newpage
\begin{appendix}
\section*{Supplementary Materials}

\startcontents[appendix]

\printcontents[appendix]{l}{0}{\setcounter{tocdepth}{2}}

\section{Related Work}\label{sec:related-work}

\paragraph{Practical TP variants.} Among backprop alternatives, the TP literature is comparatively sparse, in part due to the overhead historically associated with the method. The requirement of approximate inverses is restrictive: the canonical approach of training layerwise autoencoders introduces many iterations of feedback-weight training between forward updates, additional per-layer learning rates, and instability that worsens with depth \citep{Bartunov2018AssessingTS}. What's more, these approaches have weaker theoretical grounding, as it is difficult to characterize how close feedback functions are to the true Jacobian inverse.  Much subsequent work has focused on eliminating or simplifying these feedback weights.  \citep{Meulemans2020ATF} examines direct feedback pathways with similarities to DFA \citep{Nkland2016DirectFA}; FW-DTP \citep{Shibuya2022FixedWeightDT} pushes this further, arguing that DTP regularizes strongly enough that even a fixed random feedback matrix suffices. Our work is closer in spirit to \citep{Roulet2021TargetPV}, which leverages autograd to compute the true Jacobian and employ a linearized version of Target Propagation.
We extend this direction by embedding the inversion in a trust-region framework, separating the inverse's fidelity from its utility as a step-direction estimator. The inner-loop iterations our method incurs are spent solving a subproblem that approximates the true objective, rather than training feedback parameters that are discarded after training.

\paragraph{TP and Gauss-Newton.} The connection between TP and Gauss-Newton was first made by \citep{BengioIteratingApprox} and given more rigorous treatment by \citep{Meulemans2020ATF}, whose proof we recount in Section~\ref{sec:naive-gn}. Subsequent work has largely retreated from the GN framing. \citep{Roulet2021TargetPV} attempt to leverage the differential prescription but raise several objections to the proof in their appendix, ultimately focusing on RNN training and concluding that GN is unlikely to explain TP's empirical success. Other recent work abandons the GN connection entirely: LDRL \citep{Ernoult2022TowardsSD} aims to reproduce BP gradients directly, achieving competitive results on CIFAR-style benchmarks, while \citep{Bao2024EfficientTP} derive closed-form feedback functions via a Jacobian-matching loss. Both rely on alignment with BP gradients - positive but imperfect - which limits convergence speed by construction. We take the GN connection more seriously: solving the local subproblem properly turns TP from a method that can approximate BP into one that can potentially outpace it.

\paragraph{Second-order optimization.} 
A range of works have applied second-order machinery to training deep networks. Block-diagonal approximations of the curvature matrix are common, supported by empirical evidence \citep{Martens2015OptimizingNN, shampoo, bdapch}, and our work is closest in spirit to KFAC \citep{Martens2015OptimizingNN}, which targets the Gauss-Newton approximation and uses Levenberg-Marquardt-style damping. Our derivation differs in its treatment of this block-diagonal structure: rather than assume it, we obtain it as an upper bound on the Gauss-Newton subproblem via  Theorem~\ref{thm:majorization}.  KFAC is itself derived as a special case of our setup in Corrollary~\ref{cor:kfac}.  A connection between TP and KFAC was previously suggested in~\citep{FOOF} based on superficial similarity in their appearance.  Our work rigorously establishes this connection.

\section{What do the TP losses mean?}\label{sec:refined-tp}

The conventional presentation of Section~\ref{sec:conventional-tp} is operationally complete: it specifies how to compute targets and how to update weights. This subsection examines several choices that the presentation leaves implicit, and that bear on the theoretical status of the local subproblem. We surface them because they motivate the framing adopted in subsequent sections (and in particular, the surrogate loss of Section~\ref{sec:loss}).

First, we consider why Equation~\eqref{eq:infer-head-target} makes for a sensible target.  In backpropagation, the weights are the principal optimization variable.  Under backpropagation, thee gradients are computed via the chain rule, and the weight update aggregates the per-sample loss gradients:
\begin{equation}
    [\Delta W^{(l)}_{\text{BP}}]_{ij} = \sum_n \frac{\partial \mathcal{L}_n}{\partial W^{(\ell)}_{ij}}; \quad \frac{\partial \mathcal{L}_n}{\partial W^{(\ell)}_{ij}} = \Big(\sum_k \frac{\partial \mathcal{L}_n}{\partial z^{(\mathrm{out})}_k} J^{(\mathrm{out},\ell+1)}_{ki}(x_n)\Big)\, u^{(\ell)}_j(x_n). \label{eq:per-sample-gradient}
\end{equation}
Target propagation takes a different approach, promoting the intermediate pre-activations $z^{(\ell)}_n$ as principal.  To this end, note that the per-sample gradient vanishes as $\|\partial \mathcal{L}_n / \partial z^{(\mathrm{out})}\| \to 0$.  
A gradient descent step on the outputs $z^{(\mathrm{out})}$, as in Equation~\ref{eq:infer-head-target}, is one way to drive this.  

With the output targets specified, we turn towards the hidden layers.  As these are connected to the output by an invertible map, the natural choice for an intermediate target is the inverse image.  We can make this intuition more rigorous as follows: with the downstream targets $z^{(out)}_n$ fixed, a natural loss function for $z^{(\ell)}_n$ is the following:
\begin{equation}
    \mathcal{L}^{(\ell)}_n = ||F^{(out,\ell)}(z^{(\ell)}_n) -\hat{z}^{(out)}_n||^2
\end{equation}
Note, by comparison with \eqref{eq:per-sample-gradient}, that $\nabla_{z^{(\ell)}} \mathcal{L}_n^{(l)} =  \nabla_{z^{(\ell)}} \mathcal{L}_n$.  If $F^{(out,\ell)}$ has a corresponding inverse $[F^{(out,\ell)}]^{-1}$, the above may be re-written:
\begin{equation}
    \mathcal{L}^{(\ell)}_n = ||F^{(out,\ell)}(z^{(\ell)}_n) - F^{(out,\ell)} \circ [F^{(out,\ell)}]^{-1}(\hat{z}^{(out)}_n) ||^2
\end{equation}
If the composite function $F^{(L,\ell)}$ is $C^{(L,\ell)}$-Lipschitz, then we may bound the above loss as follows:
\begin{equation}
    \mathcal{L}^{(\ell)}_n \leq C^{(L,\ell+1)} ||z^{(\ell)}_n - [F^{(out,\ell)}]^{-1}(\hat{z}^{(out)}_n)||^2 \label{eq:tp-lipschitz-bound}
\end{equation}
We recognize the norm on the right-hand side as $\beta \mathcal{L}_{\text{TP}}$ - implying that the local TP loss, for $\beta=1$ is an upper bound on the BP loss for the targets.

\section{Proofs and Extended Remarks}

\subsection{Proof of Theorem~\ref{thm:uniform-decoupling-property}: Norm-Induced Layer Separation}\label{appendix:proof-of-uniform-norm}

Recall the setup of the multi-layer subproblem from Section~\ref{sec:multi-layer}: let $J_n$ denote the stacked Jacobian and  $\Delta z_n$ the stacked perturbation, as follows:
\begin{equation}
  J_n \coloneq \big[\, I \;\big|\; J^{(\mathrm{out}, L)}_n \;\big|\; \cdots \;\big|\; J^{(\mathrm{out}, 1)}_n \,\big]; \quad \Delta z_n \coloneq \big[\, \Delta z^{(\mathrm{out})\top}_n \;\big|\; \Delta z^{(L)\top}_n \;\big|\; \cdots \;\big|\; \Delta z^{(1)\top}_n \,\big]^\top
\end{equation}
The regularised, multi-layer subproblem is given by:
\begin{equation}
    m_n(\Delta z_n; \lambda) = \;\big\| r_n + H_n J_n \Delta z_n \big\|_{M_n}^2 + \tfrac{1}{2}\lambda \sum_{\ell=1}^{L+1} \big(\Delta z^{(\ell)}_n\big)^\top B_n^{(\ell)} \Delta z^{(\ell)}_n 
\end{equation}
In the derivation to come, we will require the following definition and lemma:

\begin{definition}[$S$-Regularised Residual] 
\label{def:s-regularised-residual}
Given the residual $r_n \in \mathbb{R}^{C}$, output Hessian $H_n \in \mathbb{R}^{C \times C}$, output metric $M_n \in \mathbb{R}^{C\times C}$, Lagrange multiplier $\lambda$, and a choice of PD matrix $S_n \in \mathbb{R}^{C\times C}$, the $S$-regularised residual is:
\begin{equation}
   \tilde r_n(\lambda;S_n) = \big( H_n^{\top} M_n H_n + \lambda S_n^{-1} \big)^{-1} H_n^{\top} M_n r \label{eq:regularised-residual}
\end{equation}
\end{definition}

\begin{lemma}[Properties of the full Jacobian]
\label{lem:surjectivity-of-J} 
$J_n$ is surjective; furthermore, for any $B_n \succ 0$, the matrix $S_n = J_n B_n^{-1} J_n^{\top} \in \mathbb{R}^{K \times K}$ is positive definite.
\end{lemma}
\begin{proof}
    For any $\Delta z_n \neq 0$, the first block of $J_n^{\top}\Delta z$ is simply $\Delta z^{(\mathrm{out})}_n$, due to the identity matrix in \ref{eq:stacked-vars}.  Hence, $J_n^{\top} \Delta z_n \neq 0$, and thus $\Delta z_n^{\top} S_n \Delta z_n = \lVert B^{-1/2}_n J_n^{\top} \Delta z_n \rVert^{2} > 0$.
\end{proof}

We are now in a position to prove the theorem.  Taking the gradient of \ref{eq:regularised-multilayer-subproblem} with respect to $\Delta z_n$, we obtain the normal equations:
\begin{equation}
    H_n^{\top}M_nH_nJ_n \Delta z^{\star}_n + \lambda B_n \Delta z^{\star}_n  = -H_n^{\top}M_nr_n \label{eq:b-normal-equations}
\end{equation}
Left multiplying by $J_nB^{-1}_n$:
\begin{equation}
    S_n H_n^{\top} M_n H_n (J \Delta z^{\star}_n) + \lambda (J \Delta z^{\star}_n) = - S_n H^{\top}_n M_n r
\end{equation}
Then, by Lemma\ref{lem:surjectivity-of-J}, $S_n^{-1}$ exists, and we have (using Definition~\ref{def:s-regularised-residual}):
\begin{equation}
    J_n \Delta z_n^{\star} = \tilde r_n(\lambda;S_n)
\end{equation}
Accordingly, any choice of metric $B_n$ is sufficient to separate the layers.  However, in what follows, we will see that a diagonal $B_n$ holds additional advantages.  To this end, re-examine the normal equations in \ref{eq:b-normal-equations}, and consider the ansatz $\Delta z^{\star}_n = - B^{-1}_n J^{\top}_n w$ for some $w \in \mathbb{R}^C$.  Substituting into the left-hand-side of \ref{eq:b-normal-equations}:
\begin{align}
    \big( J^{\top}_n H_n^{\top} M_n H_n J_n + \lambda B_n \big) \big( -B_n^{-1} J_n^{\top} w \big) &= - J_n^{\top} H_n^{\top} M_n H_n S_n w - \lambda J_n^{\top} w \\
    &= - J_n^{\top} \big( H_n^{\top} M_n H_n + \lambda S_n^{-1} \big) S_n w
\end{align}
This equals the right side of \ref{eq:b-normal-equations} when$S_n w = \tilde r_n(\lambda;S_n)$. Therefore, an alternative (and by uniqueness, equivalent) expression for the minimiser is:
\begin{equation}
    \Delta z^{\star}_n = - B_n^{-1} J_n^{\top} S_n^{-1} \tilde r_n(\lambda;S_n)
\end{equation}
$J_n^{\top} S_n^{-1} \tilde r_n(\lambda;S_n)$ is the backpropagated regularised target, and $B_n^{-1}$ acts as a preconditioner.  Block diagonality of $B_n$ ensures this feasible to compute. The computation of $S_n$ also simplifies:
\begin{equation}
    S = [B^{(\mathrm{out})}]^{-1} + \sum_{\ell} C^{(\ell)} \qquad C^{(\ell)} = J^{(\ell)} [B^{(\ell)}]^{-1} J^{(\ell)\top} \succeq 0 \label{eq:expansion-of-s-for-block-diag}
\end{equation}
To proceed, observe that, if we consider perturbations $\Delta z_n \sim \mathcal{N}(0, B^{-1})$,  the matrix $S_n = J_n B_n^{-1} J_n^{\top}$ represents the covariance of $J_n\Delta z_n$ - ie the pushforward of the prior covariance at the output.  Thus, beyond numerical convenience, the choice $S=I$ suggests every target-space direction equally likely a priori, hence we refer to it as the \textbf{egalitarian prior}.

In the invertible case, it is clear that the following choice achieves $S=I$:
\begin{equation}
    B^{(\ell)}_n = \frac{1}{\alpha_n^{(\ell)}} J^{(\ell)\top}_n J^{(\ell)}_n; \quad \alpha_n^{(\ell)} > 0,\; \sum_{\ell} \alpha_n^{(\ell)} = 1
\end{equation}
Any choice of $\alpha_n^{(\ell)}$ on the simplex will work.  The TP literature commonly makes the uniform choice $\alpha_n^{(\ell)} = (L+1)^{-1}$.  When the layers are not invertible, we include a small regularisation term as follows:
\begin{equation}
    \tilde{B}^{(\ell)} = \frac{1}{\alpha^{(\ell)}_n} \big( J_n^{(\ell)\top} J_n^{(\ell)} + \epsilon I \big) \label{eq:choice-of-B-when-not-invertible}
\end{equation}

\subsection{Remark: Handling Non-Invertible Output Layers}\label{appendix:proofs:non-invertible}

A standard deep neural network has a characteristic shape: let $D$ denote the input dimension, $K$ the dimension of the intermediate activations, and $C$ the output dimension, with the common ordering $D > K \gg C$. The first layer is non-invertible (it maps $D$-dimensional inputs to $K$-dimensional activations), but this poses no difficulty for our derivation, since we never propagate corrective targets to the input. The output layer is the genuine concern: $W^{(\mathrm{out})} \in \mathbb{R}^{K \times C}$ collapses the $K$-dimensional activation space onto $C$ outputs, with $K \gg C$ implying a large null space. The output map deletes substantial information about the activations, and we must specify how to handle this asymmetry.

Our strategy is to make the deleted information manifest. Decompose the output weight as $W^{(\mathrm{out})} = U S V^\top$ via the singular value decomposition, with $S = [S_\parallel, \,0]$ partitioned into the active singular values and the trailing zero block. The columns of $V$ split correspondingly into $V = [V_\parallel, V_\perp]$: components of the input lying in $\operatorname{col}(V_\parallel)$ are mapped to the output, while components in $\operatorname{col}(V_\perp)$ are mapped to zero. The latter are precisely the directions in activation space that the output map deletes.

We make these directions explicit by augmenting the forward map:
\begin{equation}
  \begin{pmatrix} z^{(\mathrm{out})}_n \\ c_n \end{pmatrix}
  =
  \begin{pmatrix} U S_\parallel & 0 \\ 0 & I \end{pmatrix}
  V^\top u^{(L)}_n + o b^\top,
  \qquad
  o = \begin{pmatrix} 1_\parallel \\ 0 \end{pmatrix},
\end{equation}
where $c_n$ records the components that the original output map discards, and $o$ is a mask that suppresses the bias on the augmented coordinates. The augmented map is invertible by construction: $V$ is orthogonal, the upper block $U S_\parallel$ is the original output map, and the lower-block identity preserves the previously-deleted directions.

The corresponding inverse problem for the activation perturbation is then
\begin{equation}
  \begin{pmatrix} \Delta z^{(\mathrm{out})}_n \\ \Delta c^{(\mathrm{out})}_n \end{pmatrix}
  =
  \begin{pmatrix} U S_\parallel & 0 \\ 0 & I \end{pmatrix}
  V^\top \Delta u^{(L)}_n + o b^\top.
\end{equation}

The solution to this inverse problem is
\begin{equation}
  \Delta u^{(L)}_n = W^+ \big(\Delta z^{(\mathrm{out})}_n - b\big) + V_\perp \Delta c^{(\mathrm{out})}_n,
\end{equation}
where $W^+$ denotes the pseudoinverse of $W^{(\mathrm{out})}$. Accommodating non-invertible layers thus reduces to specifying $\Delta c^{(\mathrm{out})}_n$ - a prescription for how to update the components of the activation that the loss does not directly constrain. We adopt the simplest choice, $\Delta c^{(\mathrm{out})}_n = 0$, equivalent to using the pseudoinverse of the original (unaugmented) forward map: in directions where the loss offers no corrective feedback, leave the activation unchanged.

\subsection{Classical TP as Projected Gradient Descent}\label{appendix:proof:tp-as-pgd}

Differentiating the per-datapoint reconstruction loss with respect to $W^{(\ell)}$ gives
\begin{equation}
      \frac{\partial}{\partial W^{(\ell)}} \big\| W^{(\ell)\top} u^{(\ell)}_n + b^{(\ell)} - \hat{z}^{(\ell+1)}_n \big\|^2 
  \;=\; -2\, u^{(\ell)}_n \big( \hat{z}^{(\ell+1)}_n - W^{(\ell)\top} u^{(\ell)}_n - b^{(\ell)} \big)^{\!\top}.
\end{equation}
Since $z^{(\ell+1)}_n = W^{(\ell)\top} u^{(\ell)}_n + b^{(\ell)}$ and $\hat{z}^{(\ell+1)}_n = z^{(\ell+1)}_n + \Delta z^{(\ell+1)}_n$, the residual collapses to $\Delta z^{(\ell+1)}_n$. Summing over $n$ and folding the bias into $\tilde{u}$ yields, in matrix form,
\begin{equation}
    \Delta W^{(\ell)}_{\mathrm{TP}} = 2\, \tilde{U}^{(\ell)\top} \Delta Z^{(\ell+1)}.
\end{equation}
The induced pre-activation perturbation is
\begin{equation}
    \Delta\bar{Z}^{(\ell+1)}_{\textrm{PGD}}
  = \eta \, \tilde{U}^{(\ell)} \Delta W^{(\ell)}_{\mathrm{TP}}
  = 2\eta \, \tilde{U}^{(\ell)} \tilde{U}^{(\ell)\top} \Delta Z^{(\ell+1)}, \label{eq:classical-tp-as-pgd}
\end{equation}
which lies in $\mathrm{col}(\tilde{U}^{(\ell)})$ and therefore satisfies $\Pi^{(\ell)} \Delta \bar{Z}^{(\ell+1)}_{\textrm{PGD}} = 0$ by Lemma~\ref{lem:feasibility}.

\subsection{Proof of Lemma~\ref{lemma:stacked-perturbation-contraint}: Feasibility of stacked perturbations}\label{sec:shared-feasibility}

For each datapoint, stack the perturbations of all layers,
\begin{equation}
\Delta z_n = \big( \Delta z_n^{(L+1)}, \Delta z_n^{(L)}, \dots, \Delta z_n^{(1)} \big)
\in \mathbb{R}^{D}, \qquad D = \sum_{\ell} K_{\ell},
\end{equation}
and stack these across the batch, $\Delta \bm{z} = (\Delta z_1, \dots, \Delta z_N)
\in \mathbb{R}^{ND}$. In this ordering the curvature is block-diagonal,
$\bm{G} = \bigoplus_n G_n$ with $G_n = J_n^{\top} H_n^{\top} M_n H_n J_n$. Let
$E_{\ell} \in \mathbb{R}^{K_{\ell} \times D}$ select the block of layer $\ell$, so
that $\Delta z_n^{(\ell)} = E_{\ell} \Delta z_n$ and
$\mathrm{vec}_r(\Delta Z^{(\ell)}) = (I_N \otimes E_{\ell}) \Delta \bm{z}$. The
constraint of Lemma~\ref{lem:feasibility} for layer $\ell$ then reads
\begin{equation}
\big( \Pi^{(\ell-1)} \otimes I_{K_{\ell}} \big)\big( I_N \otimes E_{\ell} \big)
\Delta \bm{z} = \big( \Pi^{(\ell-1)} \otimes E_{\ell} \big) \Delta \bm{z} = 0 .
\end{equation}
Since the $E_{\ell}$ select disjoint blocks, $E_{\ell} E_{\ell'}^{\top} = 0$ for
$\ell \ne \ell'$ and $\sum_{\ell} E_{\ell}^{\top} E_{\ell} = I_D$, the operators
$\Pi^{(\ell-1)} \otimes E_{\ell}^{\top} E_{\ell}$ are mutually orthogonal
projectors, and their sum
\begin{equation}
\bm{\Pi} = \sum_{\ell} \Pi^{(\ell-1)} \otimes E_{\ell}^{\top} E_{\ell}
\end{equation}
is itself an orthogonal projector on $\mathbb{R}^{ND}$. The constraints for all
layers are therefore equivalent to the single condition
$\bm{\Pi} \, \Delta \bm{z} = 0$.
\subsection{Proof of Theorem~\ref{thm:full-joint-subproblem}: KKT Conditions for the Constrained Subproblem}
\label{app:kkt}

This appendix collects the KKT analysis of constrained quadratic minimization with linear equality constraints, and applies it to the multi-layer subproblem of Theorem~\ref{thm:full-joint-subproblem}. The generic result is standard \citep{nocedal2006numerical}; we state and prove it here for completeness, since it is invoked at several points in the main text.

\begin{lemma}[KKT solution for equality-constrained quadratics]
\label{lem:kkt-generic}
Let $H \in \mathbb{R}^{m \times m}$ be symmetric positive definite, $D \in \mathbb{R}^{q \times m}$ have full row rank, and $x^\star \in \mathbb{R}^m$ be given. The unique minimizer of
\begin{equation}
  \min_{x \in \mathbb{R}^m} \;(x - x^\star)^\top H (x - x^\star)
  \quad \text{subject to} \quad D x = 0
  \label{eq:generic-quadratic}
\end{equation}
is
\begin{equation}
  x = (I - H^{-1} D^\top \big(D H^{-1} D^\top\big)^{-1} D) x^\star.
  \label{eq:schur-projection}
\end{equation}
The matrix $D H^{-1} D^\top$ is the Schur complement of the KKT system; it is invertible under the stated assumptions.
\end{lemma}

\begin{proof}
The Lagrangian of \eqref{eq:generic-quadratic} is
\[
  \mathcal{L}(x, \nu) = (x - x^\star)^\top H (x - x^\star) + \nu^\top D x.
\]
Stationarity with respect to $x$ gives
\[
  2H(x - x^\star) + D^\top \nu = 0
  \;\Longrightarrow\;
  x = x^\star - \tfrac{1}{2} H^{-1} D^\top \nu.
\]
Substituting into the feasibility constraint $D x = 0$:
\[
  D x^\star - \tfrac{1}{2} D H^{-1} D^\top \nu = 0
  \;\Longrightarrow\;
  \nu = 2 \big(D H^{-1} D^\top\big)^{-1} D x^\star,
\]
where invertibility of $D H^{-1} D^\top$ follows from $H \succ 0$ and $D$ having full row rank. Substituting back yields \eqref{eq:schur-projection}. Uniqueness follows from strict convexity of the objective on the affine feasible set.
\end{proof}

\begin{remark}[Geometric interpretation]
\label{rem:kkt-geometry}
The right-hand side of \eqref{eq:schur-projection} is the projection of $x^\star$ onto $\ker D$ in the inner product $\langle u, v \rangle_H = u^\top H v$. When $H = I$, this reduces to the standard Euclidean projection $x = x^\star - D^\top (DD^\top)^{-1} D x^\star$. For $H \neq I$, the projection is in the metric induced by $H$, and is the unique minimizer of the $H$-weighted distance to $x^\star$ over the feasible set.
\end{remark}

\subsection{Proof of Theorem~\ref{thm:majorization}: Simplicial Decomposition of the Global Subproblem}\label{appendix:proof:simplectic-solutions}

\paragraph{Motivation:} Recall the unregularised multi-layer subproblem from~\ref{eq:multi-layer-target-prop}:
\begin{equation}
  m_n(\Delta z_n) = \;\big\| r_n + H_n J_n \Delta z_n \big\|_{M_n}^2
\end{equation}
Let $\alpha_n$ be a set of layer-weight which lie in the interior of an $(L + 1)$ simplex, such that:
\begin{equation}
    \forall \ell\qquad \alpha^{(\ell)}_n > 0, \qquad \sum_{\ell=1}^{L+1} \alpha^{(\ell)}_n = 1
\end{equation}
It's clear from substitution that any layer perturbations $\Delta z^{(\ell)}_n$ satisfying
\begin{equation}
  H_n J^{(\ell)}_n \Delta z^{(\ell)}_n(\alpha^{(\ell)}_n) = -\alpha^{(\ell)}_n r_n,
  \qquad \ell = 1, \ldots, L+1,
  \label{eq:alpha-parameterization}
\end{equation}
will reduce $m_n$ to 0.  The difficulty lies when we have non-zero $\lambda$ and we minimize the $m_n(\Delta z_n; \lambda)$ of \ref{thm:uniform-decoupling-property}.  As we will see, this will complicate the relationship between $\alpha^{(\ell)}_n$ and $\Delta z^{(\ell)}_n$.

\begin{lemma}[A Majoriser on the GN Subproblem]
\label{lem:majoriser-on-lm}
Let $\alpha^{(\ell)}_n$ lie on the simplex.  Then the following is a majoriser of the un-regularised subproblem in Theorem~\ref{thm:uniform-decoupling-property}:
\begin{equation}
    \hat{m}_n(\Delta z_n; \alpha_n) = \sum_{\ell} \Big\| \sqrt{\alpha^{(\ell)}_n}\, r_n
  + \tfrac{1}{\sqrt{\alpha^{(\ell)}_n}} H_nJ^{(\ell)}_n \Delta z^{(\ell)}_n \Big\|_{M_n}^2 
\end{equation}

\end{lemma}
\begin{proof}

\textbf{Bounded above}\hspace{1em}Because $\sum_{\ell} \alpha^{(\ell)}_n =1$, we may write $r_n = \sum_\ell \alpha^{(\ell)} r_n$ and regroup the argument of the norm as a convex combination:
\begin{equation}
    r_n + \sum_{\ell=1}^{L+1} H_nJ^{(\ell)}_n \Delta z^{(\ell)}_n = \sum_{\ell=1}^{L+1} \alpha^{(\ell)}_n \left( r_n + (\alpha^{(\ell)}_n)^{-1} H_n J^{(\ell)}_n \Delta z^{(\ell)}_n \right) 
\end{equation}
The map $v \mapsto \lVert v \rVert_{M_n}^2$ is convex (and strictly so when $M \succ 0$) so we may apply Jensen's inequality as follows:
\begin{equation}
    m(\Delta z) \leq \frac{1}{2} \sum_{\ell=1}^{L+1} \alpha^{(\ell)}_n \left\lVert r_n + (\alpha^{(\ell)}_n)^{-1} H_nJ^{(\ell)}_n \Delta z^{(\ell)}_n \right\rVert_{M_n}^2  \coloneqq \frac{1}{2} \sum_{\ell=1}^{L+1} \hat{m}^{(\ell)}_n(\Delta z^{(\ell)}_n;\alpha^{(\ell)}_n)
\end{equation}
Thus, the coupling is removed.  Expanding each term:
\begin{equation}
    \frac{\alpha^{(\ell)}_n}{2} \left\lVert r_n + (\alpha^{(\ell)}_n)^{-1} H_n J^{(\ell)}_n \Delta z^{(\ell)}_n \right\rVert_{M_n}^2 \equiv \frac{\alpha^{(\ell)}_n}{2} \lVert r \rVert_{M_n}^2 + g^{(\ell)\top}_n \Delta z^{(\ell)}_n + \frac{1}{2\alpha^{(\ell)}_n} \Delta z^{(\ell)\top}_n G_n^{(\ell)} \Delta z^{(\ell)}_n \label{eq:expanded-local-loss-unregularised}
\end{equation}
Then, summing over $\ell$, using the fact that $\sum_\ell \alpha^{(\ell)} = 1$:
\begin{align}
    \hat{m}_n(\Delta z_n; \alpha_n) &= \frac{1}{2} \lVert r \rVert_{M_n}^2 + \sum_{\ell=1}^{L+1} \left[ g^{(\ell)\top}_n \Delta z^{(\ell)}_n + \frac{1}{2\alpha^{(\ell)}} \Delta z^{(\ell)\top}_n G_n^{(\ell)}(\lambda) \Delta z^{(\ell)}_n \right] \\
    &\geq m(\Delta z) 
\end{align}

\paragraph{Tangency in value}
\begin{equation}
     \hat{m}_n(0; \alpha_n) = \frac{1}{2} \sum_{\ell=1}^{L+1} \alpha^{(\ell)}_n \lVert r \rVert_{M_n}^2 = \frac{1}{2} \lVert r \rVert_{M_n}^2 = m(0) 
\end{equation}
\paragraph{Tangency in gradient}

Differentiating the $\ell$-th term of $\hat{m}_n(\Delta z_n; \alpha_n)$:
\begin{align}
     \nabla_{z^{(\ell)}_n} \hat{m}_n(\Delta z_n; \alpha_n) &= \alpha^{(\ell)}_n \cdot (\alpha^{(\ell)}_n)^{-1} J_n^{(\ell)\top} H_n^{\top} M_n \left( r_n + (\alpha_n^{(\ell)})^{-1} H_nJ_n^{(\ell)} \Delta z^{(\ell)}_n \right) \\
     &= g^{(\ell)}_n + \frac{1}{\alpha^{(\ell)}_n} G_n^{(\ell)} \Delta z^{(\ell)}_n
\end{align}
At the origin, this is simply:
\begin{equation}
    \nabla_{z^{(\ell)}_n} \hat{m}_n(0; \alpha_n) = g^{(\ell)}_n = \nabla_{z^{(\ell)}_n} m_n(0)
\end{equation}
\end{proof}

Now that the layers may be treated separately, we introduce the trust-region framework through a Levenberg-Marquardt penalty.  This leads to the following definition:

\begin{definition}[Regularised Layer-Local Subproblem]
\label{def:regularised-surrogate-objective}  Given the layer weight $\alpha^{(\ell)}_n$, Lagrange multiplier $\lambda$, and the block-diagonal metric $B_n$ of Theorem~\ref{thm:uniform-decoupling-property}, the layer-local subproblem is:
\begin{equation}
    Q^{(\ell)}_n(\Delta z^{(\ell)}_n; \lambda, \alpha^{(\ell)}_n) = \hat{m}^{(\ell)}_n(\Delta z^{(\ell)}_n;\alpha^{(\ell)}_n) + \tfrac{\lambda}{2} \Delta z^{(\ell)\top}_n B^{(\ell)}_n \Delta z^{(\ell)}_n 
\end{equation}
\end{definition}

By completing the square in the above definition, we may read off additional objects of note:
\begin{align}
    Q^{(\ell)}_n(\Delta z^{(\ell)}_n; \lambda, \alpha^{(\ell)}_n) &= \Phi^{(\ell)}_n(\lambda, \alpha^{(\ell)}_n) + \tfrac{1}{2}||\Delta z^{(\ell)}_n + \Delta z^{(\ell)\star}_n||^2_{\tilde{G}^{(\ell)}_n(\lambda, \alpha^{(\ell)}_n)} \\
\tilde{G}^{(\ell)}_n(\lambda, \alpha^{(\ell)}_n) &\coloneqq \frac{1}{\alpha^{(\ell)}_n}G^{(\ell)}_n + \lambda B^{(\ell)}_n &\quad \mathrm{Curvature}\\
\Delta z^{(\ell)\star}_n(\lambda, \alpha^{(\ell)}_n) &\coloneqq -[\tilde{G}^{(\ell)}_n(\lambda, \alpha^{(\ell)}_n)]^{-1}g^{(\ell)}_n &\quad \mathrm{Minimiser}  \label{eq:q-def-minimiser}
    \\ \Phi^{(\ell)}_n(\lambda, \alpha^{(\ell)}_n)&\coloneqq  \alpha^{(\ell)}_n ||r||^2_{M_n} - g^{(\ell)\top}_n [\tilde{G}^{(\ell)}_n(\lambda, \alpha^{(\ell)}_n)]^{-1}g^{(\ell)}_n &\quad \mathrm{Residual} \label{eq:q-def-residual}
\end{align}

\begin{remark}[Interpretation of the $\alpha^{(\ell)}_n$]: The layer weights enable us to distribute a single trust-region parameter $\lambda$ across the different layers and datapoints. 
\end{remark}

The residual $\Phi^{(\ell)}_n(\lambda, \alpha^{(\ell)}_n)$ holds particular relevance: it is error that remains unaccounted for by the minimiser $\Delta z^{(\ell)\star}_n(\lambda, \alpha^{(\ell)}_n)$, and so the optimal perturbations will be found by choosing $\alpha^{(\ell)}_n$ in the simplex that minimise $\Phi(\lambda, \alpha) = \sum_{\ell} \Phi^{(\ell)}_n(\lambda, \alpha^{(\ell)}_n)$.  Unfortunately, the relationship between $\alpha^{(\ell)}_n$ and $\Delta z^{(\ell)\star}_n(\lambda, \alpha^{(\ell)}_n$ is too complex to find a closed-form solution, and so we must resort to iterative methods.  To this end, we prove the following lemma:

\begin{lemma}[Convexity of the Regularised Layer-Local Subproblem]
\label{lem:convexity-of-Q}
For $\lambda > 0$, $B^{(\ell)}_n \succ 0$, and $G^{(\ell)}_n \succeq 0$,
$Q^{(\ell)}_n(\Delta z^{(\ell)}_n;\lambda,\alpha^{(\ell)}_n)$ is jointly convex in
$(\Delta z^{(\ell)}_n, \alpha^{(\ell)}_n)$ on $\mathbb{R}^K \times (0,\infty)$.
\end{lemma}
\begin{proof}
From \ref{eq:expanded-local-loss-unregularised}, $Q$ is a sum over three terms:
\begin{align}
    \mathrm{(i)} \quad &: \quad \frac{\alpha^{(\ell)}_n}{2} \lVert r \rVert_{M_n}^2 + g^{(\ell)\top}_n \Delta z^{(\ell)}_n\\
    \mathrm{(ii)} \quad &: \quad \frac{1}{2\alpha^{(\ell)}_n} \Delta z^{(\ell)\top}_n G_n^{(\ell)} \Delta z^{(\ell)}_n \\
    \mathrm{(iii)} \quad &: \quad   \frac{\lambda}{2} \Delta z^{(\ell)\top}_n B_n^{(\ell)} \Delta z^{(\ell)}_n
\end{align}

Each term is jointly convex in $(\Delta z^{(\ell)}_n, \alpha^{(\ell)}_n)$:
\begin{itemize}
  \item \textbf{(i)} is affine.
  \item \textbf{(ii)} For $\alpha^{(\ell)}_n > 0$,
        \begin{equation}
          \frac{1}{2\alpha^{(\ell)}_n}\, \Delta z^{(\ell)\top}_n G^{(\ell)}_n \Delta z^{(\ell)}_n
          = \sup_{y}\Bigl( y^\top (G^{(\ell)}_n)^{1/2} \Delta z^{(\ell)}_n
            - \tfrac{1}{2}\alpha^{(\ell)}_n\, y^\top y \Bigr),
        \end{equation}
        which is a pointwise supremum of functions affine in $(\Delta z^{(\ell)}_n, \alpha^{(\ell)}_n)$.
  \item \textbf{(iii)} is a convex quadratic in $\Delta z^{(\ell)}_n$ alone, since $B^{(\ell)}_n \succ 0$.
\end{itemize}
Thus, since the sum of jointly convex functions is jointly convex, $Q$ is jointly convex.
\end{proof}

The joint convexity of $Q^{(\ell)}_n(\Delta z^{(\ell)}_n;\lambda,\alpha^{(\ell)}_n)$ suggests that a simple alternating minimisation scheme can be used to find the optimal $\alpha_n$ and $\Delta z^{(\ell)}_n$.  The optimal $\Delta z^{(\ell)}_n$ for fixed $\alpha^{(\ell)}_n$ may be read off from the completed-square in \ref{eq:q-def-minimiser}.  For fixed $\Delta z^{(\ell)}_n$, we make use of the following Lemma:

\begin{lemma}[Optimal layer weights given fixed pre-activation perturbations]
\label{lem:optimal-alpha-from-deltaz}
The optimal $\alpha^{(\ell)}_n$ as a function of $\Delta z^{(\ell)}_n$ is given by:
\begin{equation}
\alpha^{(\ell)\star}_n(\Delta z_n) = \frac{\lVert \Delta z^{(\ell)}_n \rVert_{G^{(\ell)}_n}}{\sum_k \lVert \Delta z^{k}_n \rVert_{G^{(k)}_n}}
\end{equation}

\end{lemma}
\begin{proof}
Let $\mu$ be a Lagrange multiplier enforcing $\sum_{\ell} \alpha_n^{(\ell)} = 1$, and consider the Lagrangian:
    \begin{equation}
        L_{\alpha} = Q^{(\ell)}_n(\Delta z^{(\ell)}_n; \lambda, \alpha^{(\ell)}_n) + \mu \bigg(\sum_{\ell} \alpha_n^{(\ell)} - 1\bigg)
    \end{equation}
Taking the derivative with respect to $\alpha_n^{(\ell)}$, we have:
\begin{equation}
    \frac{\partial L_{\alpha}}{\partial \alpha_n^{(\ell)}} = -\frac{1}{2(\alpha^{(\ell)}_n)^2} \Delta z^{(\ell)\top}_n G_n^{(\ell)} \Delta z^{(\ell)}_n + \mu = 0
\end{equation}
Thus:
\begin{equation}
\alpha^{(\ell)}_n = \frac{1}{\sqrt{2\mu}} ||\Delta z^{(\ell)}_n||_{G_n^{(\ell)} }
\end{equation}
Now, performing the sum over $\ell$:
\begin{align}
    \frac{1}{\sqrt{2\mu}} \sum_{\ell}||\Delta z^{(\ell)}_n||_{G_n^{(\ell)} } &= 1 \\
    \implies \sqrt{2\mu} &= \sum_{\ell}||\Delta z^{(\ell)}_n||_{G_n^{(\ell)} }
\end{align}
Hence, finally:
\begin{equation}
    \alpha^{(\ell)}_n(\Delta z) = \frac{||\Delta z^{(\ell)}_n||_{G_n^{(\ell)}}}{\sum_k ||\Delta z^{(k)}_n||_{G_n^{(k)}}}
\end{equation}
\end{proof}

We now prove that alternating minimisation can solve the subproblem:

\begin{lemma}[Convergence of Alternating Minimisation]
\label{lem:convergence-alternating-minimisation}
The sequence below converges to the minimum of $Q_n(\Delta z_n; \lambda, \alpha_n)$:
\begin{equation}
    \text{(A)}\quad \Delta z^{(\ell)}_n \leftarrow -\alpha^{(\ell)}_n \big( G^{(\ell)}_n + \alpha^{(\ell)}_n \lambda B^{(\ell)} \big)^{-1} g^{(\ell)}_n, \qquad \text{(B)}\quad \alpha^{(\ell)}_n \leftarrow \frac{\lVert \Delta z^{(\ell)}_n \rVert_{G^{(\ell)}_n}}{\sum_k\lVert \Delta z^{(k)}_n \rVert_{G^{(k)}_n}}
\end{equation}
\end{lemma}
\begin{proof}
Step (A) is the exact minimiser of $Q_n(\Delta z_n; \lambda, \alpha_n)$: over $\Delta z_n$ for fixed $\alpha_n$ . Step (B) is the exact minimizer of $Q_n(\Delta z_n; \lambda, \alpha_n)$: over $\alpha_n$ in the simplex for fixed $\Delta z_n$. Denoting $\Delta z_n(t), \alpha_n(t)$ the $t$-th iterates, the algorithm is monotonically non-increasing.
\begin{equation}
    Q_n(\Delta z_n(t+1); \lambda, \alpha_n(t+1)) \le Q_n(\Delta z_n(t+1); \lambda, \alpha_n(t)) \leq Q_n(\Delta z_n(t); \lambda, \alpha_n(t))
\end{equation}

The sequence is bounded below by $\min_{\Delta z_n} m(\Delta z_n) \ge 0$, and so it converges.
\end{proof}

\subsection{Proof of Theorem~\ref{thm:inscribed}: Single-parameter coupling of the two trust regions}\label{sec:proof-thm-inscribed}

Throughout, $\langle X, Y \rangle = \operatorname{tr}(X^{\top} Y)$, $u_n^{\top}$
are the rows of $\tilde U$, and we write $\Delta z_n = \Delta W^{\top} u_n$ for
the pre-activation change induced on datapoint $n$.

\emph{Inscription.} Both $q_z$ and $q_w$ are nonnegative. If
$q_z(\Delta W) + q_w(\Delta W) \le 1$, then $q_z(\Delta W) \le 1 - q_w(\Delta W)
\le 1$, and likewise $q_w(\Delta W) \le 1$, so
$\{ q_z + q_w \le 1 \} \subseteq \{ q_z \le 1 \} \cap \{ q_w \le 1 \}$.
Minimising the local objective over the inscribed ellipsoid is a convex problem
in which $\Delta W = 0$ is strictly feasible; by Slater's condition strong
duality holds, and its Lagrangian with multiplier $\lambda / 2$ is, up to a
constant, the penalised objective $L_\lambda$.

\emph{The penalty.} By the definitions in \eqref{eq:mu-from-lambda} and
\eqref{eq:two-forms},
\begin{equation}
\frac{\lambda}{2} \big( q_z + q_w \big)
= \frac{\mu_z}{2} \lVert \tilde U \Delta W \rVert_{F}^{2}
+ \frac{\mu_J}{2} \lVert \Delta W \bar{\tilde G}^{1/2} \rVert_{F}^{2},
\end{equation}
which gives the first line of \eqref{eq:single-lambda}.

\emph{Completing the square.} Let $g_n = J_n^{\top} H_n^{\top} M_n r_n$. Expanding the
$n$-th summand and adding its share $\frac{\mu_z}{2} \lVert \Delta z_n
\rVert^{2}$ of the pre-activation penalty,
\begin{align}
\frac{1}{2} \Big\lVert \sqrt{\alpha_n} \, r_n + \frac{1}{\sqrt{\alpha_n}} J_n
\Delta z_n \Big\rVert_{M_n}^{2} + \frac{\mu_z}{2} \lVert \Delta z_n \rVert^{2}
&= \frac{\alpha_n}{2} r_n^{\top} M_n r_n + g_n^{\top} \Delta z_n
+ \frac{1}{2} \Delta z_n^{\top} \tilde G_n \Delta z_n ,
\end{align}
since the cross term carries the factor $2 \sqrt{\alpha_n} \alpha_n^{-1/2} = 2$.
This quadratic is minimised at $\Delta z_n^{\star} = -\tilde G_n^{-1} g_n$, and
because $\tilde G_n \Delta z_n^{\star} = -g_n$,
\begin{equation}
g_n^{\top} \Delta z_n + \frac{1}{2} \Delta z_n^{\top} \tilde G_n \Delta z_n
= \frac{1}{2} \big( \Delta z_n - \Delta z_n^{\star} \big)^{\top} \tilde G_n
\big( \Delta z_n - \Delta z_n^{\star} \big)
- \frac{1}{2} \Delta z_n^{\star\top} \tilde G_n \Delta z_n^{\star} .
\end{equation}
Summing over $n$ and collecting the terms independent of $\Delta W$ into the
constant gives the second line of \eqref{eq:single-lambda}.

\emph{Stationarity.} Write $e_n = \Delta W^{\top} u_n - \Delta z_n^{\star}$.
Perturbing $\Delta W \to \Delta W + V$ changes $e_n$ by $V^{\top} u_n$, so the
first-order change in $\frac{1}{2} \sum_n e_n^{\top} \tilde G_n e_n$ is
$\langle V, \sum_n u_n e_n^{\top} \tilde G_n \rangle$. By symmetry of
$\bar{\tilde G}$, the first-order change in
$\frac{\mu_J}{2} \operatorname{tr}(\Delta W^{\top} \Delta W \bar{\tilde G})$ is
$\langle V, \mu_J \Delta W \bar{\tilde G} \rangle$. Setting the gradient to zero,
\begin{equation}
\sum_n u_n u_n^{\top} \Delta W \tilde G_n + \mu_J \Delta W \bar{\tilde G}
= \sum_n u_n \Delta z_n^{\star\top} \tilde G_n = -\sum_n u_n g_n^{\top},
\end{equation}
where the last equality uses $\tilde G_n \Delta z_n^{\star} = -g_n$ and the
symmetry of $\tilde G_n$. This is \eqref{eq:reduced-system-reg}.

\emph{Uniqueness.} For any $V$,
\begin{equation}
\langle V, \mathcal{G}_{\lambda}(V) \rangle
= \sum_n (V^{\top} u_n)^{\top} \tilde G_n (V^{\top} u_n)
+ \mu_J \operatorname{tr}(V^{\top} V \bar{\tilde G})
\ge \mu_J \mu_z \lVert V \rVert_{F}^{2},
\end{equation}
since $\tilde G_n \succeq 0$ and $\bar{\tilde G} = \bar G + \mu_z I \succeq
\mu_z I$. The same expression is symmetric in its two arguments, so
$\mathcal{G}_{\lambda}$ is symmetric positive definite in the Frobenius inner
product. Hence $L_\lambda$ is strictly convex, and its unique minimiser is the
unique solution of \eqref{eq:reduced-system-reg}.

\subsection{Proof of Corrollary~\ref{cor:kfac}: Exact recovery of KFAC}

Substituting $\tilde G_n = \bar{\tilde G}$ into \eqref{eq:reduced-system-reg},
\begin{equation}
\mathcal{G}_{\lambda}(V) = \Big( \sum_n u_n u_n^{\top} \Big) V \bar{\tilde G}
+ \mu_J V \bar{\tilde G} = \big( A + \mu_J I \big) V \bar{\tilde G}.
\end{equation}
Both factors are symmetric positive definite, so the system
$(A + \mu_J I) \Delta W \bar{\tilde G} = -\nabla_W L$ is solved by
left-multiplying by $(A + \mu_J I)^{-1}$ and right-multiplying by
$\bar{\tilde G}^{-1}$. The right-hand side is unchanged by the substitution, since
it is the true gradient for any choice of per-sample metric.

\section{Additional Experimental Details}\label{appendix:more-exp-details}

\subsection{Settings}

In this section, we recount additional experimental details that were not described in the main body of the paper.  Our models are implemented using PyTorch \citep{pytorch}, and trained using NVIDIA RTX A6000 48GB GPUs.

\paragraph{General Training Details} All second order methods were trained over 15 epochs, and all first-order methods (including the TP variants from the literature) were given 30 epochs.  We reserve 10\% of the training set as a validation set.  The test errors are reported at the epoch with the best validation error.  In all experiments, we denote the batch-size as fractions (so that a batch-size of 0.1 translates to 10\% of the training dataset).  Hyperparameters are tuned via Bayesian optimization on \texttt{wandb} against validation accuracy, with three seeds per configuration; we report the test error of the configuration with the best validation accuracy.  

\paragraph{Trust Region Settings} The exact trust region algorithm is given as a listing in Algorithm~\ref{alg:tr}, where \textsc{Solve} is replaced by the strategy specified by KFAC, GDTP, EucTP, CGTP (no WD), or CGTP.  We use the common values that are default in standard packages \citep{2020SciPy-NMeth}: a stringency ($\eta_{\text{acc}}$) of 0.15, acceptance/success ($\eta_{\text{lo}}/\eta_{\text{high}}$) thresholds of 0.25 and 0.75, a growth factor $\gamma_{\uparrow}$ of 2 and a shrink factor $\gamma_{\downarrow}$ of 0.25.  The regularization strength is clipped to lie within 1e-5 and 1024.  

\paragraph{Subproblem norm} Each of the trust regions methods take, as subproblem norm, $M_n = H_n^{-1}$.  For softmax cross-entropy classification problems (like those considered in the experiments), $H_n$ is not invertible.  We combat this by adding a small, constant penalty on the magnitude of the model outputs, set to $\sigma = 0.001$, which amounts to Tikhonov regularisation $H_n \to H_n + \sigma I$.

\paragraph{Preconditioned Conjugate Gradient} We perform an ablation study (Table~\ref{tab:simplex-scheme-ablation}) incorporating preconditioners into our CG algorithm.  Preconditioners are desirable because they can enable faster convergence, enabling us to get competitive results with a smaller number in cg iterations.  We precondition with the averaged operator $\mathcal{P} = (A+\mu_JI) \otimes \bar{\tilde G}$ rather than a per-sample alternative. Beyond its Kronecker structure, this choice is the one consistent with the null-space elimination of Section~\ref{sec:cg-alg} as it commutes with with the constraint operator.  The general preconditioned CG algorithm is shared in Algorithm~\ref{alg:cg}. Ultimately, however, we find that using preconditioners results in worse performance, and the gap is related to the choice of simplex scheme (Supp~\ref{sec:simple-simplex-schemes}).  

\paragraph{Custom KFAC implementation} KFAC here is our own implementation.  Per Corrollary~\ref{cor:kfac}, KFAC is exactly recovered from CGTP when the per-sample Jacobians are replaced by their batch mean.  To ensure a fair comparison, we use the same trust-region adapation as CGTP rather than periodic-inversion schedule of \citep{Martens2015OptimizingNN}. Likewise, we invert every step. This is more expensive per-step than a scheduled inversion but negligible relative to a network backward pass. We use a uniform simplex scheme $\alpha = L^{-1}$ to decouple the layer subproblems.

\subsection{Compute Requirements}

\begin{table}[htbp]
\centering
\begin{tabular}{lcccc}
\hline
\textbf{Method} & \textbf{Batch size} & \textbf{Peak VRAM (mb)} & \textbf{FLOPs / step} & \textbf{Total Wall Clock (10 epochs)} \\
\hline
BP+ADAM    & 0.01 & 103.49  & 1.08 GFLOP & 6m 11s \\
KFAC    & 0.1 & 660.40  & 93.8 GFLOP & 6m 24s \\
CGTP & 0.1  & 3210.10 & 1.16 TFLOP & 6m 47s \\
\hline
\end{tabular}
\caption{Resource comparison between BP+ADAM and CGTP}
\label{tab:resource-comparison}
\end{table}

While we emphasise that our implementation is an unoptimised proof-of-concept, we include runtime statistics in Table~\ref{tab:resource-comparison}, computed over a profiling run (10 epochs, batch sizes as listed) to measure per-step cost.  Per step, CGTP uses roughly 12x the FLOPs of KFAC and 1000x those of Adam, because each outer iteration runs up to 100 CG iterations. These inner solves are independent across layers, so in a layer-parallel implementation they need not add to wall-clock time in proportion to depth; our implementation solves layers sequentially.  Finally, other factors already favour our approach, e.g. reduced host-device data movement, since a batch is loaded to the GPU once and reused across all inner-loop iterations.

\subsection{Simplex Schemes}\label{sec:simple-simplex-schemes}

Theorem~\ref{thm:majorization} describes how the multi-layer subproblem may be decomposed into non-interacting parts, and how alternating minimisation to find the optimal set of layer weights that maximises the expected decrease.  Unfortunately, while updating $\Delta z_n$ for fixed $\alpha_n$ can be done locally, updating the $\alpha_n$ requires a global syncronisation step across layers.  To this end, we perform an aggressively truncated alternating minimisation sequence: we initialise $\Delta z_n$ according to one of the strategies shared below, and then use this to take one $\alpha$-step.  These are the alphas that then enter into our CG solver, and an ablation is shown in Table~\ref{tab:simplex-scheme-ablation}.

The strategies are as follows:
\begin{itemize}
    \item \textbf{Uniform}: No $\alpha$-step.  We simply initialise $\alpha^{(\ell)}_n = (L+1)^{-1}$.
    \item \textbf{Unit} $\Delta z$: We consider $||\Delta z^{(\ell)}_n|| = 1$, and approximate $||\Delta z^{(\ell)}_n||_{G^{(\ell)}_n} \approx ||G^{(\ell)}_n||_F$.
    \item \textbf{Large} $\lambda$: In the limit of large $\lambda$, the regulariser $B^{(\ell)}_n$ dominates, so that $\Delta z^{(\ell)}_n \approx -\tfrac{1}{\lambda}[B^{(\ell)}_n]^{-1}g^{(\ell)}_n$.
\end{itemize}

\begin{table}[h]
    \centering
    
    \caption{Numerical comparison of the test error and number of iterations to reach a given test error for the different simplex schemes. Hyperparameters for each algorithm were optimised for validation error.}
\begin{tabular}{rlll}
\toprule
Preconditionr & Simplex Scheme & Test Score & milestone@0.13 \\
\midrule
False & Large $\lambda$ & 0.1085 ± 0.0016 & 6.6667 ± 0.5774 \\
False & Uniform & 0.1126 ± 0.0019 & 9.6667 ± 2.0817 \\
False & Unit $||\Delta z||$ & 0.1107 ± 0.0008 & 12.6667 ± 2.0817 \\
True & Large $\lambda$ & 0.1108 ± 0.0014 & 7.6667 ± 0.5774 \\
True & Uniform & 0.1240 ± 0.0027 & 9.0000 ± 1.7321 \\
True & Unit $||\Delta z||$ & 0.1188 ± 0.0032 & 7.3333 ± 1.5275 \\
\bottomrule
\end{tabular}
    \label{tab:simplex-scheme-ablation}
\end{table}
\FloatBarrier   

\newpage
\clearpage
\subsection{Ablation Studies}

\noindent\textit{This subsection contains:}
\begin{itemize}\itemsep0pt
    \item Figure~\ref{fig:ablation-batch-size}: compares test error between ADAM, KFAC, and CGTP at equal batch size.
    \item Figure~\ref{fig:cg_iters_batch_size}: grid over batch size and max CG iterations, validation error (MNIST). Batch size matters most.
  \item Figure~\ref{fig:depth-ablation}: test error against network depth (Fashion-MNIST), supporting the stability-across-depth claim.
  \item Figure~\ref{fig:vary-batch-size}: training and test curves for varying batch size, max CG iterations fixed at 100 (Fashion-MNIST).
  \item Figure~\ref{fig:vary-max-iter}: training and test curves for varying max CG iterations at batch size 0.1 (Fashion-MNIST). Small training gains, little effect on test error once large enough batch-size reached.
\end{itemize}

\begin{figure}[H]
    \centering
    \includegraphics[width=0.85\linewidth]{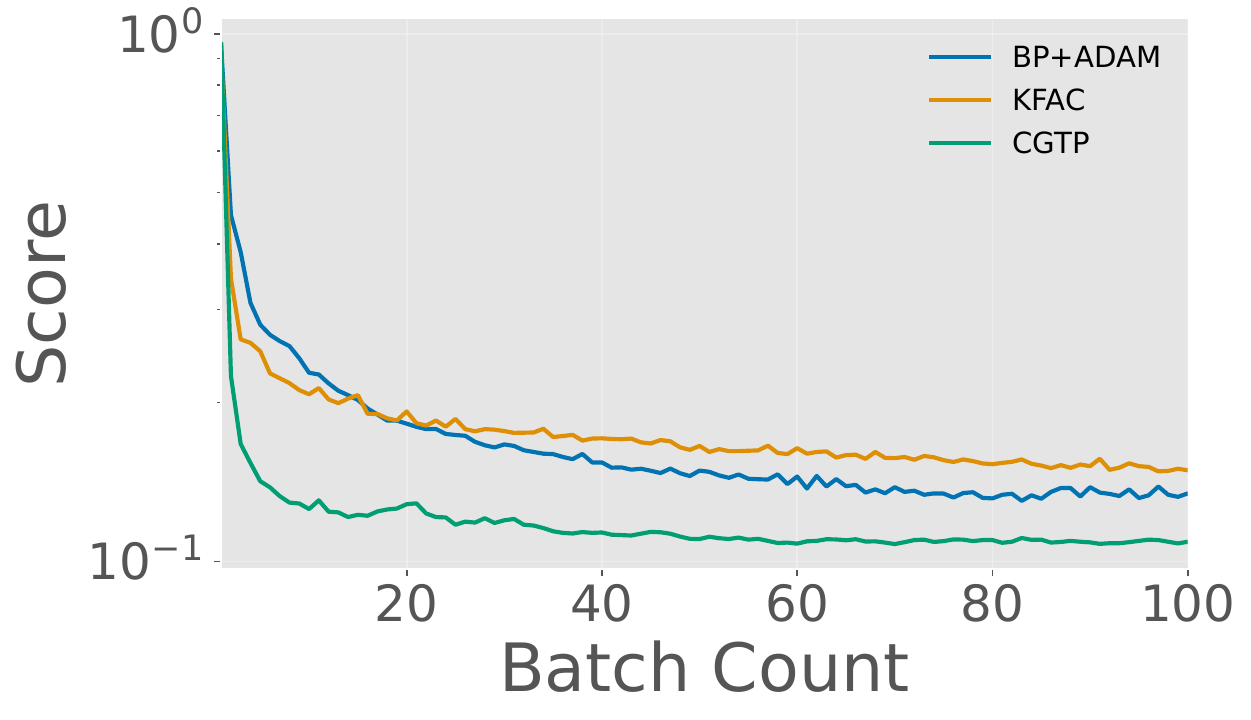}
    \caption{Comparing convergence of different optimisers at batch size 0.1 (Fashion-MNIST).  CGTP hits the 0.13 score milestone after 7 batches.  BP hits it at 50 batches, and KFAC does not reach it.}
    \label{fig:ablation-batch-size}
\end{figure}

\begin{figure}[htbp]
     \centering
     \includegraphics[width=0.8\textwidth]{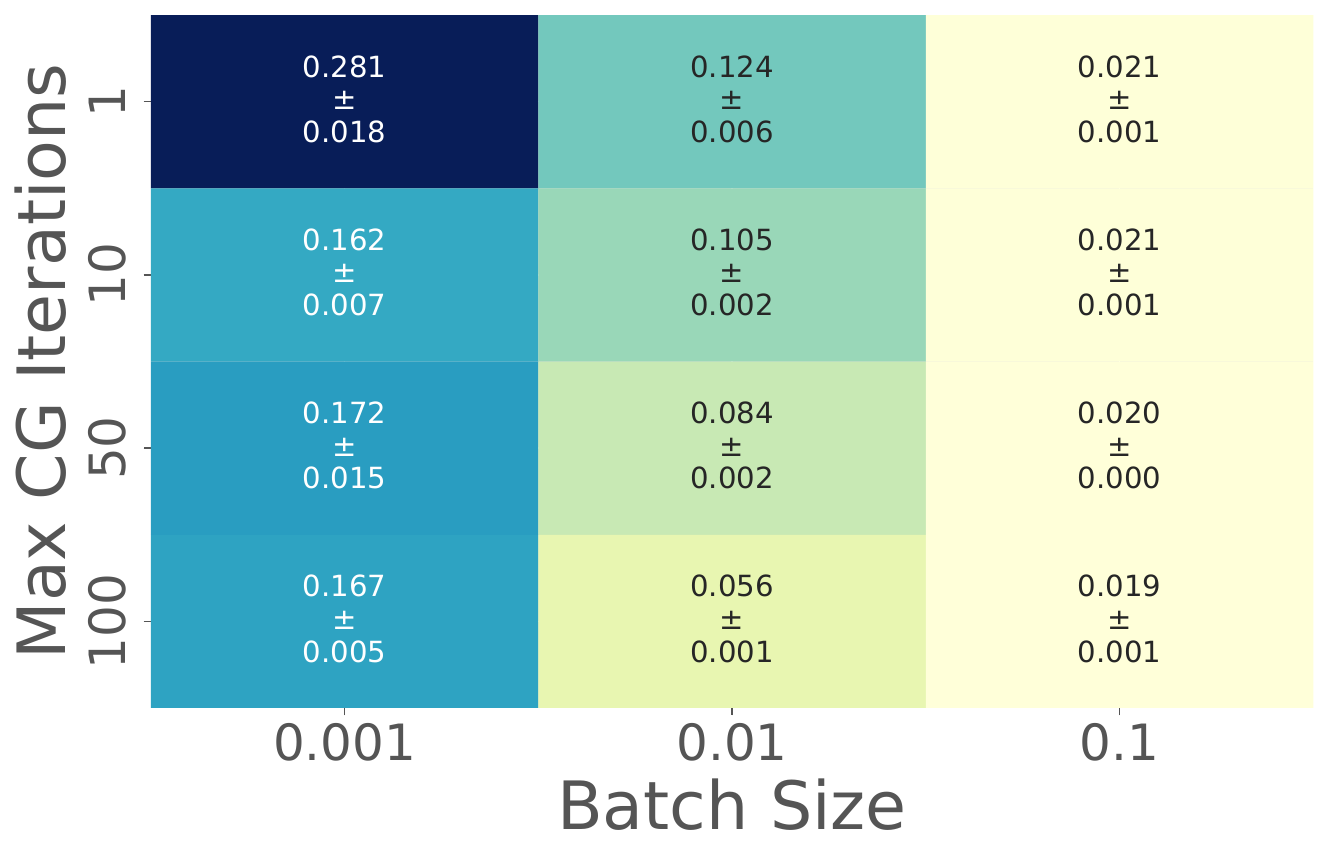}

         \label{fig:right}
         \caption{Results of a grid search over max CG iterations and batch size on validation error (MNIST dataset).  By far the batch size has the biggest impact on model performance, due to its importance in estimating curvature (training curves may be seen in Figure\ref{fig:vary-batch-size}).  However, even with a poor curvature estimate, increasing the max cg iter can improve performance - though its impact is diminished once sufficient batch size is reached.}
     \label{fig:cg_iters_batch_size}
\end{figure}

\begin{figure}[htbp]
     \centering

     \includegraphics[width=0.5\textwidth]{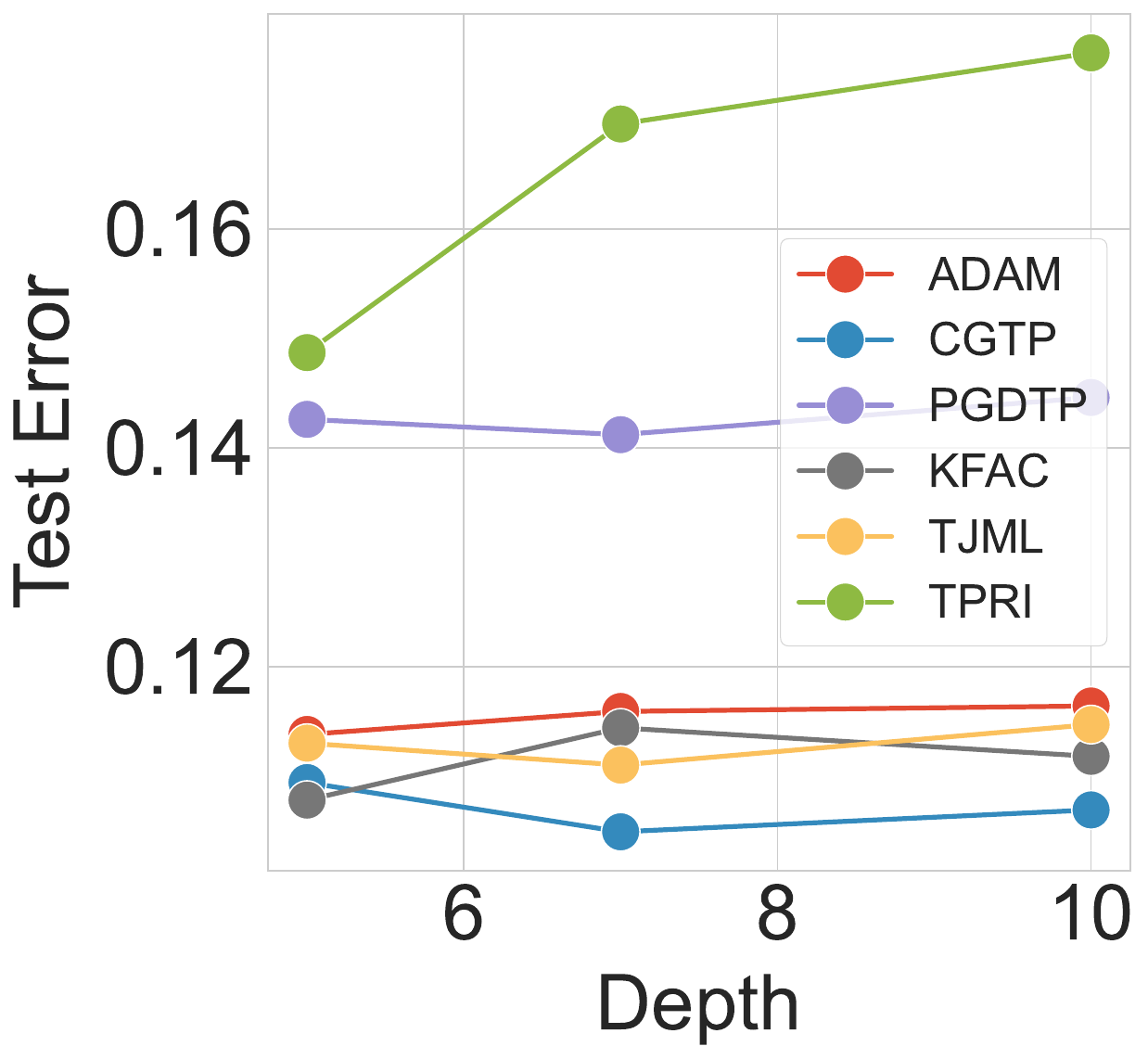}

     \caption{A graph of the impact of network depth on test error (Fashion-MNIST). In comparison with methods such as TPRI, CGTP is stable across depth.  This is likely due to the difference between the regularised composite Jacobian, and the product of regularised layer Jacobians.}
     \label{fig:depth-ablation}
\end{figure}

\begin{figure}
    \centering
    \includegraphics[width=0.95\linewidth]{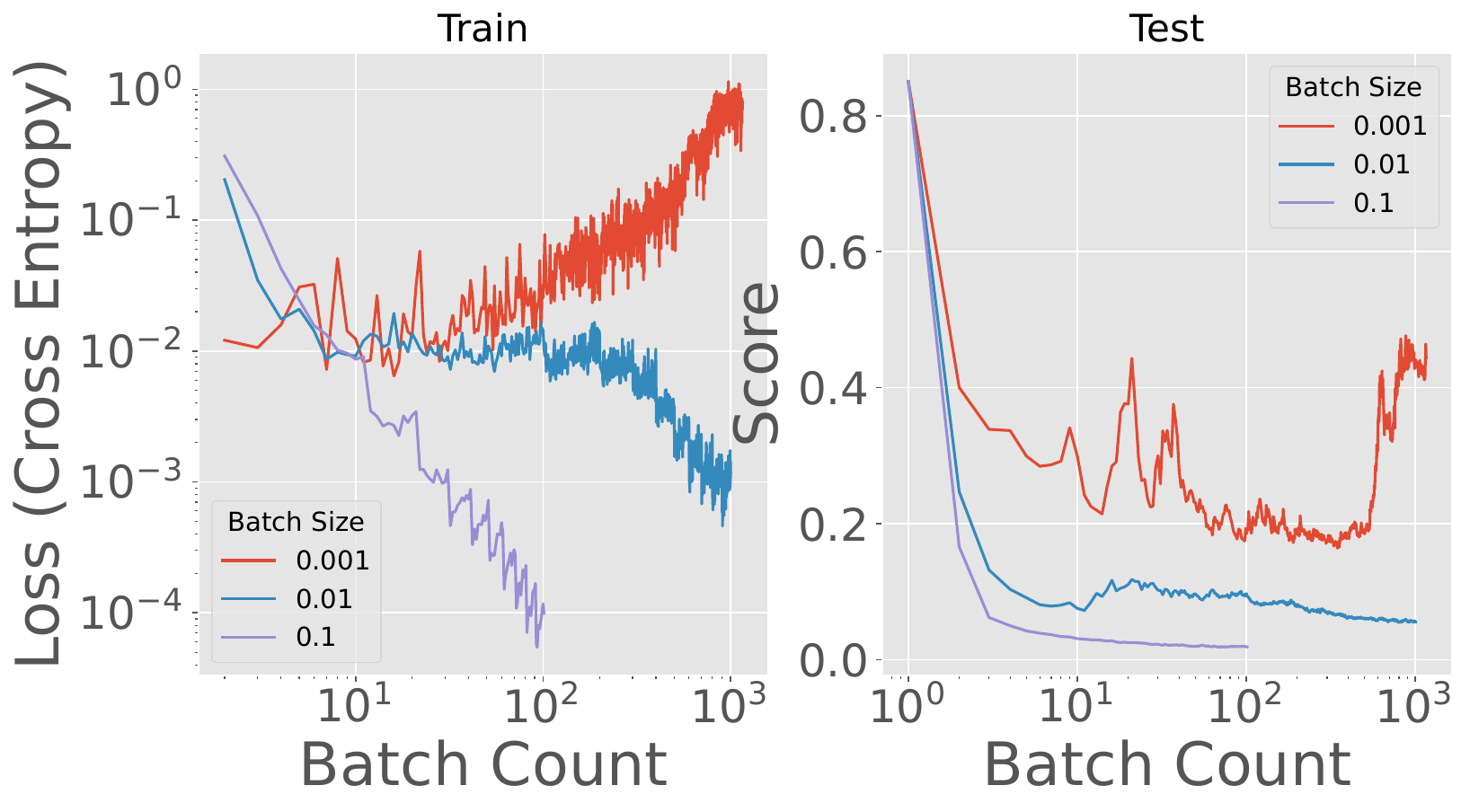}
    \caption{Varying the batch size at fixed max cg iter (Fashion MNIST): We keep the max cg iter fixed at 100, and vary the batch size for CGTP.  Larger batches result drastically accelerated learning, which is also reflected in the test error (with diminishing returns).  This highlights a known phenomenon of second order optimizers, as they need larger batch sizes to get better second order estimates of the curvature.}
    \label{fig:vary-batch-size}
\end{figure}

\begin{figure}
    \centering
    \includegraphics[width=0.95\linewidth]{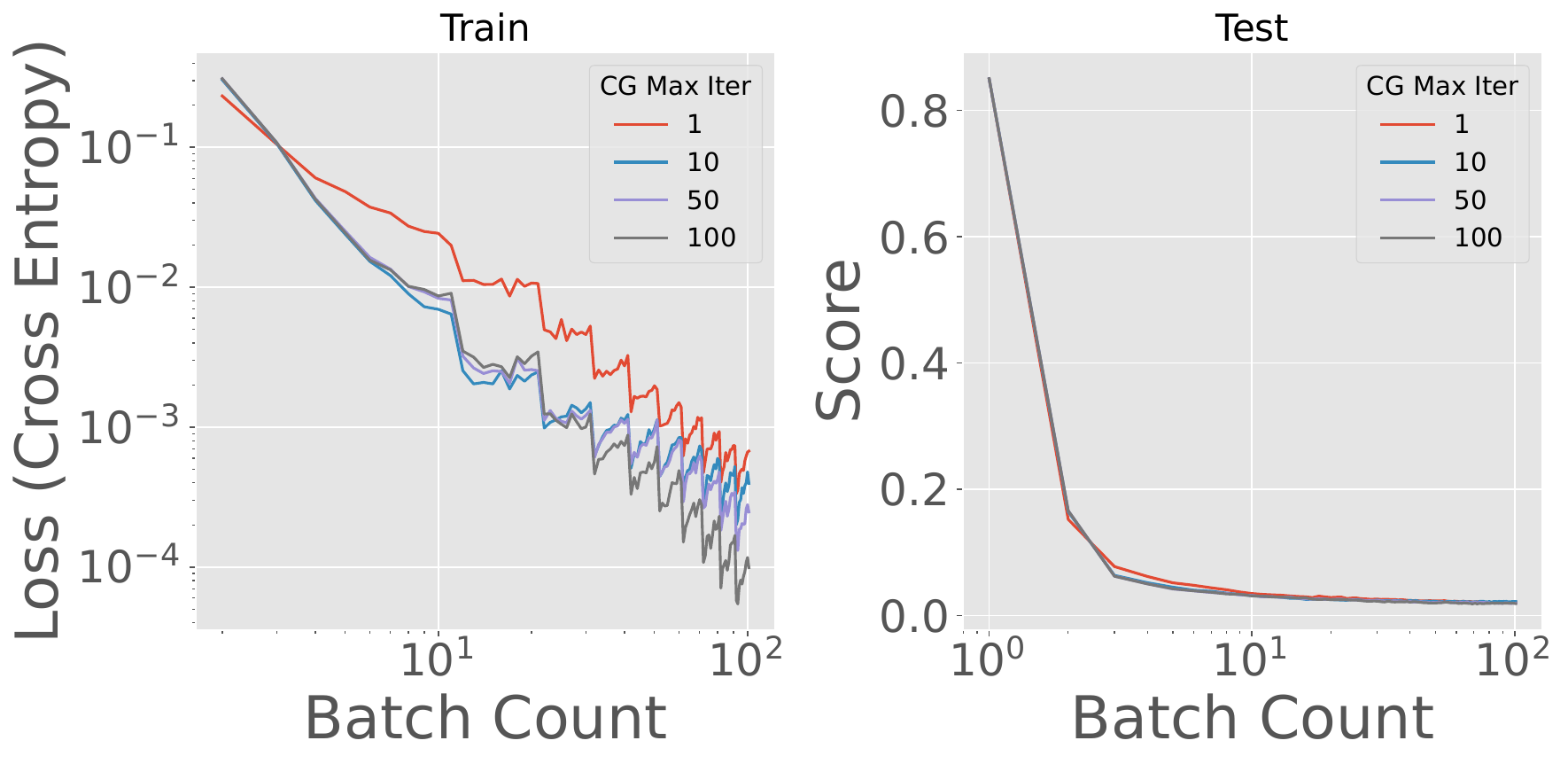}
    \caption{Varying the max cg iter at fixed batch size (Fashion-MNIST): We keep the batch sized fixed at batch size 0.1 for CGTP.  We see small improvements in terms of learning speed on the train set, but this does little for the test error.}
    \label{fig:vary-max-iter}
\end{figure}
\newpage
\newpage
\section{Algorithm Listings}
\noindent We include these listings for completeness. The trust-region update
(Algorithm~\ref{alg:tr}) and the preconditioned conjugate gradient solver
(Algorithm~\ref{alg:cg}) are standard and adapted from \citep{nocedal2006numerical}; what is specific to this work is the subproblem they
are applied to (Sections~\ref{sec:trust-region}-~\ref{sec:second-order-tp}).

\begin{algorithm}[H]
\caption{Trust-region outer loop with a pluggable subproblem solver (\textsc{Solve})}
\label{alg:tr}
\begin{algorithmic}[1]
\Require parameters $\theta$; initial damping $\lambda$; solver \textsc{Solve};
  stringency $\eta_{\mathrm{acc}}=0.15$, thresholds $\eta_{\mathrm{lo}}=0.25$, $\eta_{\mathrm{hi}}=0.75$;
  factors $\gamma_{\uparrow}=4$, $\gamma_{\downarrow}=0.25$;
  bounds $[\lambda_{\min},\lambda_{\max}]=[10^{-5},1024]$
\For{each outer iteration}
  \State Draw a batch; forward pass; one backward pass to form $r_n, H_n, J_n^{(\ell)}, g_n^{(\ell)}$
    \Comment{only global communication}
  \State $\alpha \gets \textsc{SimplexScheme}(\cdot)$ \Comment{Supp.~D.3}
  \For{each layer $\ell$ \textbf{in parallel}}
    \State $\Delta W^{(\ell-1)} \gets \textsc{Solve}(\ell;\,\lambda,\alpha)$ \Comment{GDTP, CGTP, KFAC, ...}
  \EndFor
  \State $\mathrm{pred} \gets m(0)-m(\Delta z)$, where $\Delta z^{(\ell)}=\tilde U^{(\ell-1)}\Delta W^{(\ell-1)}$
  \State $\mathrm{obsd} \gets J(\theta)-J(\theta+\Delta\theta)$
  \State $\rho \gets \mathrm{obsd}/\mathrm{pred}$ \Comment{Compute performance ratio}
  \If{$\rho>\eta_{\mathrm{acc}}$}
    \State $\theta \gets \theta+\Delta\theta$ \Comment{otherwise reject the step}
  \EndIf
  \If{$\rho<\eta_{\mathrm{lo}}$}
    \State $\lambda \gets \min(\gamma_{\uparrow}\lambda,\ \lambda_{\max})$
  \ElsIf{$\rho>\eta_{\mathrm{hi}}$}
    \State $\lambda \gets \max(\gamma_{\downarrow}\lambda,\ \lambda_{\min})$
  \EndIf
\EndFor
\end{algorithmic}
\end{algorithm}

\begin{algorithm}%
\caption{Preconditioned conjugate gradient for the layer-$\ell$ realisation problem}
\label{alg:cg}
\begin{algorithmic}[1]
\Require inputs $\tilde U^{(\ell-1)}$ with rows $u_n^{\top}$; per-sample
$J_n^{(\ell)}$, $M_n$, $r_n$, $\alpha_n^{(\ell)}$, $B_n^{(\ell)}$; trust
parameter $\lambda$; preconditioner damping $\delta \ge 0$; tolerance
$\varepsilon$; iteration budget $T$
\Ensure approximate solution $\Delta W^{(\ell-1)}$ of \eqref{eq:reduced-system}
\Statex
\Function{ApplyG}{$V$} \Comment{$\mathcal{G}^{(\ell)}_{\lambda}(V)$, matrix-free}
  \State $Y \gets \tilde U^{(\ell-1)} V$ \Comment{rows $y_n^{\top} = u_n^{\top} V$}
  \For{$n = 1, \dots, N$}
    \State $\hat y_n \gets \frac{1}{\alpha_n^{(\ell)}} J_n^{(\ell)\top} M_n J_n^{(\ell)} y_n + \lambda B_n^{(\ell)} y_n$ \Comment{one JVP, one VJP}
  \EndFor
  \State \Return $\tilde U^{(\ell-1)\top} \hat Y$
\EndFunction
\Statex
\Function{Precond}{$R$} \Comment{inverse of $(A + \delta I) \otimes \bar G$}
  \State \Return $(A + \delta I)^{-1} R \, \bar G^{-1}$ \Comment{two Cholesky solves}
\EndFunction
\Statex
\State $g_n \gets J_n^{(\ell)\top} M_n r_n$ for all $n$; \quad $\Gamma \gets -\sum_n u_n g_n^{\top}$
\State $A \gets \tilde U^{(\ell-1)\top} \tilde U^{(\ell-1)}$; \quad $\bar G \gets \frac{1}{N} \sum_n \tilde G_n^{(\ell)}(\lambda; \alpha)$
\State $\Delta z_n^{(\ell)\star} \gets -\tilde G_n^{(\ell)}(\lambda; \alpha)^{-1} g_n$ for all $n$ \Comment{targets}
\State $X \gets (A + \delta I)^{-1} \tilde U^{(\ell-1)\top} \Delta Z^{(\ell)\star}$ \Comment{warm start, or $X \gets 0$}
\State $R \gets \Gamma - \Call{ApplyG}{X}$; \quad $\hat R \gets \Call{Precond}{R}$; \quad $P \gets \hat R$; \quad $\rho \gets \langle R, \hat R \rangle$
\For{$t = 1, \dots, T$}
  \If{$\lVert R \rVert_{F} \le \varepsilon \lVert \Gamma \rVert_{F}$} \textbf{break} \EndIf
  \State $\Theta \gets \Call{ApplyG}{P}$
  \State $a \gets \rho / \langle P, \Theta \rangle$
  \State $X \gets X + a P$; \quad $R \gets R - a \Theta$
  \State $\hat R \gets \Call{Precond}{R}$; \quad $\rho' \gets \langle R, \hat R \rangle$
  \State $P \gets \hat R + (\rho' / \rho) P$; \quad $\rho \gets \rho'$
\EndFor
\State \Return $\Delta W^{(\ell-1)} \gets X$
\end{algorithmic}
\end{algorithm}

\end{appendix} 

\end{document}